\pdfoutput=1
\documentclass{article}
\usepackage{iclr2027_conference,times}
\usepackage{amsmath,amsfonts,bm}

\def\eqref#1{equation~\ref{#1}}
\def\1{\bm{1}}

\DeclareMathAlphabet{\mathsfit}{\encodingdefault}{\sfdefault}{m}{sl}
\SetMathAlphabet{\mathsfit}{bold}{\encodingdefault}{\sfdefault}{bx}{n}

\usepackage{wrapfig}
\usepackage{url}
\usepackage{graphicx}
\usepackage{float}
\usepackage{placeins}
\usepackage{amsthm}
\usepackage[ruled,vlined,linesnumbered]{algorithm2e}
\newtheorem{assumption}{Assumption}
\newtheorem{theorem}{Theorem}
\usepackage{booktabs}
\usepackage{enumerate,enumitem}
\usepackage{makecell}

\usepackage{etoc}
\usepackage{hyperref}

\setlist[itemize]{leftmargin=*}
\setlist[enumerate]{leftmargin=*}

\title{Early Signatures of Memorization in Diffusion Models via Basin Geometry and Cyclic Denoising}

\author{
\textbf{Nikhil Verma}$^{1,**}$, 
\textbf{Siddharthan Dileep}$^{1,**}$, 
\textbf{Anoop Singh}$^{2}$ \\
\textbf{Srikanth Sastry}$^{3}$, 
\textbf{Ramya Hebbalaguppe}$^{4}$, 
\textbf{Sayan Ranu}$^{1,2}$, 
\textbf{N. M. Anoop Krishnan}$^{1,5}$ \\
$^{1}$Yardi School of Artificial Intelligence, Indian Institute of Technology Delhi \\
$^{2}$Department of Computer Science and Engineering, Indian Institute of Technology Delhi \\
$^{3}$Jawaharlal Nehru Centre for Advanced Scientific Research \\
$^{4}$ TCS Research Labs\\
$^{5}$Department of Civil and Environmental Engineering, Indian Institute of Technology Delhi \\
$^{**}$Equal contribution
}
\iclrfinalcopy

\begin{document}
\maketitle
\lhead{Preprint}
\etocdepthtag.toc{main}
\begin{abstract}

Diffusion models generalize early in training and later reproduce individual training samples. Standard tests detect memorization only once one-shot generation already produces near-copies, which leaves a released model unaudited until its outputs fail. We show that memorization is encoded in the geometry of the learned energy landscape before it appears in generated samples, a state we call \emph{latent memorization}. Using score divergence and basin volume, we find that localized basins form around individual training samples and separate them from held-out samples before the first memorized sample appears, with an onset that follows the same $O(n)$ scaling as the memorization time. We probe these basins with cyclic denoising, the repeated application of partial noising followed by denoising. Under the exact empirical score, we prove that cycling started near an isolated training sample recovers it and returns to it over any finite number of cycles with high probability, with a bound controlled by the cycling noise and the separation from competing samples. In trained models, cycling recovers training images from CelebA and CIFAR-10 checkpoints whose one-shot samples contain no copies, and at a CelebA checkpoint with 0.1\% one-shot copies, 500 cycles raise the memorized fraction above 30\%. Cycling also reveals degenerate attractors that match no single training image. They are prevalent early in training and fade as training proceeds, so residence in a basin does not by itself imply memorization. These findings hold on a Gaussian mixture, CelebA, and CIFAR-10 across optimizers, architectures, noise schedules, and training-set sizes. They also extend to off-the-shelf Stable Diffusion v1.4, where the cycled conditional--unconditional divergence gap separates memorized from non-memorized prompts with an AUC of $0.944$ and a TPR of $0.866$ at $1\%$ FPR. Altogether, probing the learned landscape uncovers memorization that eludes one-shot generation. More broadly, what a diffusion model has memorized is a property of the geometry and stability of its learned distribution, and assessing it requires examining this structure rather than generated outputs alone.

\end{abstract}

\section{Introduction}
\begin{figure*}[t]
\vspace{-0.3in}
  \centering
  \includegraphics[width=\textwidth]{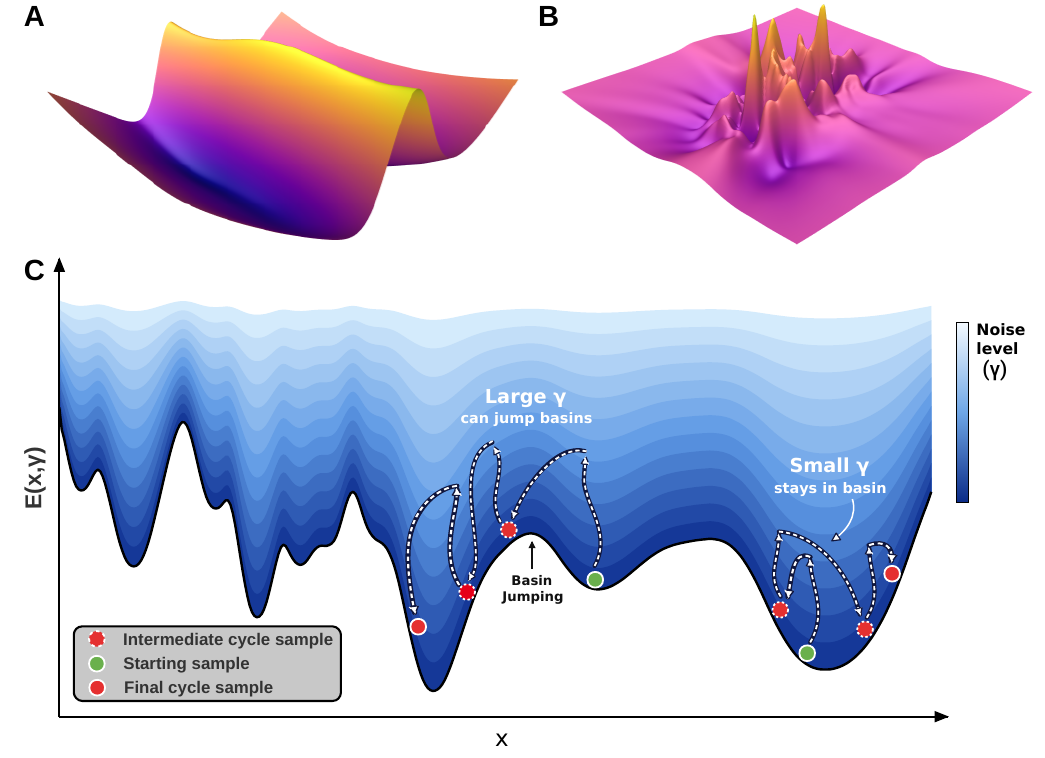}
    \vspace{-0.2in}
    \caption{%
    \textbf{Memorization forms sharp, localized basins in the learned score field.}
(\textbf{A,B}) Normalized score divergence $\widehat{\mathcal{A}}_1$ for an eight-dimensional Gaussian mixture model (GMM), on a two-dimensional slice of the sample space obtained with the projection of \citet{bihani2024low} (App.~\ref{app:landscape-viz}).
(\textbf{A}) In the generalized regime, the field is smooth, with no prominent structure around training samples.
(\textbf{B}) As memorization emerges, strongly negative regions form around training samples.
(\textbf{C}) Schematic of cyclic denoising: at low noise, trajectories remain within a basin, whereas at sufficiently large noise they jump to neighbouring basins.}
  \vspace{-0.2in}
  \label{fig:landscape_cyclic}
\end{figure*}

Diffusion models (DMs) generate high-quality and diverse samples by learning to reverse a gradual noising process \citep{pmlr-v37-sohl-dickstein15,hoDenoisingDiffusionProbabilistic2020,songScoreBasedGenerativeModeling2020}. The noising process is a stochastic differential equation (SDE) that transforms data into Gaussian noise, and its time reversal requires the \textit{score}, the gradient of the log-density of the noised data. A neural network trained to approximate this score acts as a force field that guides samples from noise back to data. The training objective that enables generalization also admits memorization. For a finite training set, the loss-minimizing score is the score of the training data convolved with Gaussian noise, and exact sampling with this score drives every trajectory onto a training point as the noise vanishes \citep{guMemorizationDiffusionModels2023}. At this optimum, the model reproduces only its training data. Practical DMs do not reach this optimum immediately.
They first generate novel samples and begin to reproduce individual training examples only after further training~\citep{bonnaireWhyDiffusionModels2025}.
Such reproduction has been documented in large image DMs \citep{10203298,somepalliUnderstandingMitigatingCopying2023}, and generate-and-filter attacks extract training images from deployed models \citep{3620237.3620531}. Memorization therefore has direct consequences for privacy, copyright, and the protection of sensitive data.

Detecting memorization in a released model is therefore a central problem for auditing DMs. A released model is a single checkpoint, usually without its training curves and often without its training set. The standard test generates samples and counts near-copies of training images, and this test registers memorization only once one-shot generation already reproduces the training set. A model may, however, encode individual training samples before any of them appears in its outputs. We call this state \emph{latent memorization}: the model has formed structure specific to individual training samples, yet one-shot generation produces no copies. Detecting it early, in the auditing sense, means detecting it from the model in hand, before copies appear in generated samples and without access to the training curve.

\textbf{Related Work.}
The generalization-to-memorization transition is well-characterized along two axes, model capacity and training-time, for models whose training is analyzed. Along the first, DMs generate novel samples when their effective capacity is insufficient to reproduce the training set \citep{yoon2023diffusion}.
Along the second, \citet{bonnaireWhyDiffusionModels2025} and \citet{faveroBiggerIsntAlways2025} identify a generalization time $\tau_{\mathrm{gen}}$, at which high-quality novel samples appear, and a memorization time $\tau_{\mathrm{mem}}$, beyond which samples reproduce training examples.
Since $\tau_{\mathrm{mem}}$ grows linearly with the number of training samples $n$, a generalization window $[\tau_{\mathrm{gen}},\tau_{\mathrm{mem}}]$ widens with $n$.
Training and test losses separate before $\tau_{\mathrm{mem}}$, but this signal requires the training curve. Energy-based analyses connect DMs to dense associative memories and show that, as the dataset grows, the learned landscape passes from attractors centred on individual training examples, through spurious attractors, to generalized states with vanishing basin volume
\citep{ambrogioniSearchDispersedMemories2023,phamMemorizationGeneralizationEmergence2025}.
Memorization is also local.
Isolated examples in sparse regions of data space are memorized first \citep{mergerLocalCoverageGoverns2026}, and generated samples drift towards the training set while the held-out loss still improves \citep{garnier-brunBiasedGeneralizationDiffusion2026}.

For trained models, geometric detectors measure the learned score through conditional--unconditional score differences, Hessian-based sharpness, directional anisotropy, or $p$-Laplacian estimates
\citep{wenDetectingExplainingMitigating2024,jeonUnderstandingMitigatingMemorization2025,asthanaDetectingMitigatingMemorization2026,brokmanIdentifyingMemorizationDiffusion2025,zotero-item-106}.
These detectors are validated on models and prompts whose memorization is already visible in generated samples.
Cyclic denoising (CD), the repeated application of partial noising followed by denoising, offers a dynamical probe of the same landscape.
\citet{kangSamplingDataChains2026} characterize how such chains mix across the data manifold at fixed noise, and \citet{sharmaCyclicDenoisingReveals2026} use CD to extract training images from fixed trained models (see App.~\ref{app:extended_rw} for further related work). 
Together, these works characterize memorization once it is visible in generated samples, or along axes that require access to training.
There is limited understanding on how the geometry underlying memorization develops before copies appear in generated samples, and about what CD reveals in that regime.
We ask whether memorization leaves geometric and dynamical signatures in the learned score before one-shot generation reproduces the training set, and to what extent these signatures can be read from the model alone.

\textbf{Contributions.}
In this work, we show that memorization is encoded in the geometry of the learned score before it is expressed by single-shot generation. We study this through the time-dependent energy landscape $E_t(x)=-\log p_t(x)$ learned by the model (Fig.~\ref{fig:landscape_cyclic}).
As training proceeds, this landscape develops localized basins around individual training samples. Within the generalization window, these basins capture little of the probability mass reached from pure noise, so one-shot sampling rarely lands in them. CD changes this picture. Each cycle perturbs a sample within a neighborhood set by the noise amplitude $\gamma$ and denoises it back, effectively leading to stable regions of the landscape characterized by local minima, for instance, a training-sample basin. Thus, repeated cycling exposes basins that a single draw may miss.
Our contributions are as follows.
\vspace{-0.05in}
\begin{itemize}
    \item \textbf{Geometric signatures of memorization.} Using score divergence and a recovery-based basin volume, evaluated on training, held-out, and generated samples across training, we show that training-sample basins form before one-shot generation reproduces those samples. The onset of these signatures scales with $\tau/n$, as does $\tau_{\mathrm{mem}}$.
    \item \textbf{Latent memorization.} Applying CD across checkpoints and noise amplitudes, we expose latent memorization inside the generalization window. CD raises the memorized fraction from at most $0.1\%$ to over $30\%$. CD also reveals degenerate attractors, stable states that match no single training image, which are prevalent near $\tau_{\mathrm{gen}}$ and vanish as training proceeds.
    \item \textbf{Theoretical proof.} Under exact stochastic reversal with the empirical score, we derive a lower bound on the probability that CD repeatedly returns to a training point, controlled by the cycling noise and by the separation of that point from competing training points.
    \item \textbf{Empirical evaluation.} We establish these signatures on a Gaussian mixture model (GMM), CelebA, and CIFAR-10 across two optimizers, three architectures, two noise schedules, and a range of training-set sizes, and show that CD detects memorized prompts in off-the-shelf Stable Diffusion v1.4 \citep{9878449} without retraining.

\end{itemize}

\vspace{-0.12in}
\section{Method}
\vspace{-0.1in}
Sampling from a DM \citep{hoDenoisingDiffusionProbabilistic2020,songScoreBasedGenerativeModeling2020} simulates probability-flow or Langevin dynamics governed by the learned score field $s_\theta(x_t,t)$, where $\theta$ denotes the learned parameters (see App.~\ref{app:prelim-diff} for the preliminaries). This score defines an energy landscape through the time-dependent energy $E_t(x)=-\log p_t(x)$, since $s_t(x)=-\nabla_x E_t(x)$. Because the score points toward increasing probability, it also points toward decreasing energy, and high-probability regions are low-energy regions. Each local minimum of $E_t$ defines the centre of a \emph{basin}, the set of points that flow toward it. In the memorization regime, individual training samples can form its own low-energy basins~\citep{phamMemorizationGeneralizationEmergence2025,ambrogioniSearchDispersedMemories2023}, which makes basin geometry a natural lens on memorization. In the remainder of the paper, we use this energy-landscape view to study memorization through the landscape learned by the model, $E_{t,\theta}(x)=-\log p_{t,\theta}(x)$.

\subsection{Geometric Signatures of the Energy Landscape}
\vspace{-0.1in}
\label{sec:geom-sig}
We characterize the learned landscape by two basin-level signatures, evaluated on training, held-out, and generated samples: score divergence $\mathcal{A}(x,t)$ and basin volume. The score divergence measures how sharply the landscape contracts around a point, and the basin volume measures the size of the neighborhood from which denoising returns to it. Evaluating them on training samples requires the training set, whereas evaluating them on generated and held-out samples requires only the model and reference data from the same distribution.

\textbf{Score divergence and its estimation.}
Near a locally convex energy minimum, the learned score
$\boldsymbol{s_\theta(x,t)}$ can draw nearby points toward the basin
centre. We characterize this net local contraction by the
divergence \textbf{$\nabla_x\cdot s_\theta(x,t)$}. For an exact score,
this divergence equals
$\operatorname{tr}(\nabla_x^2\log p_t(x))
=-\operatorname{tr}(\nabla_x^2E_t(x))$,
the negative sum of the local energy curvatures. A more
negative value near a basin centre therefore indicates
stronger local concentration. If training forms localized
basins around individual examples, the divergence near
training samples should therefore fall below that near
held-out samples before one-shot generation reproduces
the training set. For the reported measurements we use
the boundary Monte Carlo method of
\citet{brokmanIdentifyingMemorizationDiffusion2025}.
It computes the unscaled mean outward score flux on a
ball of radius $R_{\mathrm p}$ around the forward-perturbed
sample $x_t$. We denote this reported estimate by
$\boldsymbol{\mathcal A(x,t)}$ and call it
\emph{score divergence} for convenience. It estimates the
divergence averaged over the ball up to the known positive
factor $R_{\mathrm p}/d$, as detailed in
App.~\ref{app:score-divergence}. Because this estimate depends on both the magnitude and
the direction of the boundary scores, we isolate direction
through the \emph{normalized score divergence}
$\boldsymbol{\widehat{\mathcal A}(x,t)}$, the mean outward
alignment of the unit-normalized boundary scores. We use
this name for convenience. $\widehat{\mathcal A}$ lies in
$[-1,1]$ when the sampled score norms are nonzero, and
negative values indicate predominantly inward-pointing
scores. Unless stated otherwise, we evaluate both
statistics at the final denoising step ($t=1$) and denote
them by $\boldsymbol{\mathcal A_1}$ and
$\boldsymbol{\widehat{\mathcal A}_1}$.

\begin{figure*}[t]
\vspace{-0.2in}
\centering
\includegraphics[width=\textwidth]{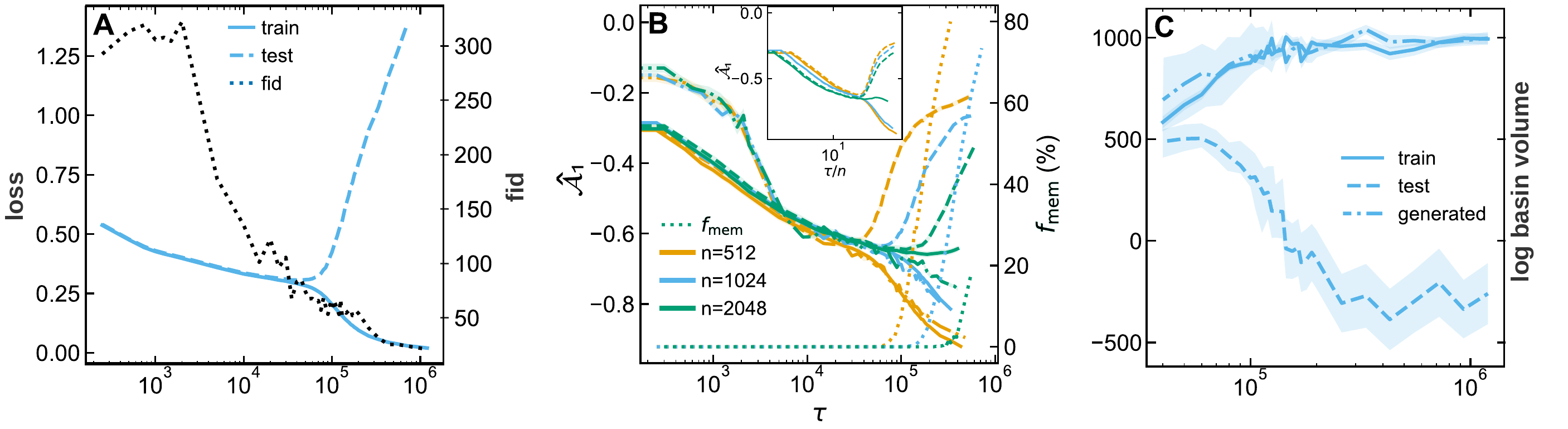}
\vspace{-0.3in}
\caption{\textbf{Geometric signatures across training on CelebA.}
(\textbf{A}) Training loss (solid), held-out loss (dashed), and FID (dotted, right axis) for $n=1024$.
(\textbf{B}) $\widehat{\mathcal{A}}_1$ at $\gamma=0.001$ and $f_{\mathrm{mem}}$ (dotted, right axis) for $n=512$ (orange), $1024$ (blue), and $2048$ (green); training (solid), held-out (dashed), and generated (dash-dot) samples. \emph{Inset:} divergence against $\tau/n$.
(\textbf{C}) Mean log basin volume for $n=1024$ over all training samples and an equal number of generated and held-out samples, with line styles as in (B); bands show one standard deviation across samples. Refer Appendix~\ref{app:additional-results} for equivalent results on Adam Optimizer}
\vspace{-0.2in}
\label{fig:training-dynamics}
\end{figure*}

\textbf{Basin volume.} We estimate the basin volume with the DDIM recovery sweep of \citet{phamMemorizationGeneralizationEmergence2025}. For a target sample $\hat{x}$, let $\gamma_c$ be the critical noise amplitude (see Section~\ref{sec:cyclic-probe}) from which deterministic DDIM denoising recovers $\hat{x}$ within a tolerance $\epsilon_{\mathrm{rec}}$, and let $x_{\gamma_c}$ be the perturbed input at that amplitude.
The distance $r_{\mathrm{rec}}(\hat{x})=\|\hat{x}-x_{\gamma_c}\|_2$ defines a recovery radius, and the basin volume is the volume of a $d$-dimensional ball with this radius,
$\boldsymbol{V_{\mathrm{basin}}(\hat{x})=\frac{\pi^{d/2}}{\Gamma(d/2+1)}r_{\mathrm{rec}}(\hat{x})^d}$.
$V_{\mathrm{basin}}$ is thus an isotropic proxy for the recovery volume and does not equal the exact basin volume. 

\vspace{-0.1in}
\subsection{Cyclic Denoising as an Energy Landscape Probe}
\vspace{-0.1in}
\label{sec:cyclic-probe}

The signatures above describe the landscape at individual points; CD probes how samples move through it. At cycle $k$, we noise the current sample to timestep $t$ with the forward marginal $\mathcal{Q}_t$ and denoise it back, $x^{(k)}\xrightarrow{\mathcal{Q}_t}x_t^{(k)}\xrightarrow{p_\theta}x^{(k+1)}$,where $p_\theta(x^{(k+1)}\mid x_t^{(k)})$ is the distribution induced by running the learned reverse process from $t$ to $0$. We parametrize each cycle by the normalized timestep $\boldsymbol{\gamma=t/T}$, which we call the \textit{cycling amplitude}.Because noise schedules vary in how they allocate noise over time, evaluating them at the same normalized timestep ($\gamma$) results in differing noise standard deviations ($\sigma_t$) (see App.~\ref{app:sigma-gamma}).

The cycling amplitude sets the scale of landscape exploration.
Small $\gamma$ produces predominantly local motion, while large $\gamma$ produces stronger perturbations that make transitions between basins more likely. A trajectory therefore tends to remain in its current basin while each perturbation stays within the region from which denoising returns it to that basin, and it escapes more often as perturbations grow (Fig.~\ref{fig:landscape_cyclic}C).To follow these dynamics, we track four quantities across cycles. The consecutive-cycle similarity $\boldsymbol{S_{\mathrm{cyc}}^{(k)}=\cos\!\bigl(\phi_{\mathrm{LPIPS}}(x^{(k)}),\phi_{\mathrm{LPIPS}}(x^{(k-1)})\bigr)}$, where $\phi_{\mathrm{LPIPS}}$ denotes the LPIPS feature representation \citep{Zhang_2018_CVPR}, measures whether a trajectory has settled.
We say that a trajectory resides in a basin from its residence onset, the first sustained crossing of $95\%$ of the rise in $S_{\mathrm{cyc}}^{(k)}$ from its early-cycle level to its late-cycle level, and we refer to the canonical image of that basin as its attractor. Beyond residence, we make this identification through the memorized fraction $\boldsymbol{f_{\mathrm{mem}}}$, which under CD is the fraction of trajectories whose state at cycle $k$ satisfies the same nearest-neighbour criterion used for one-shot generation.
When a trajectory resides in a basin whose attractor matches no single training image under this criterion, we call the attractor \emph{degenerate}. Finally, we track $\boldsymbol{\mathcal{A}_1}$ and $\boldsymbol{\widehat{\mathcal{A}}_1}$ along each trajectory to relate these dynamics to the local geometry of the basins it visits.

\vspace{-0.1in}
\subsection{Theoretical Analysis of Basin Residence Under Cyclic Denoising}
\vspace{-0.1in}
\label{sec:method-theory}

Our experiments suggest that once a cycling trajectory reaches the basin of a memorized training image, subsequent cycles can repeatedly recover that image. To isolate this mechanism, we analyze CD under exact stochastic reversal of the empirical training distribution, following the exact empirical-score idealization\citep{biroliDynamicalRegimesDiffusion2024}.
\begin{assumption}[Exact empirical reversal]
\label{assump:exact-score}
The clean distribution is $N^{-1}\sum_{j=1}^N\delta_{x_j}$ for distinct training points $x_1,\ldots,x_N\in\mathbb R^d$, with $N\ge2$. Each cycle noises its current state $x$ to $U=x+\sqrt q\,Z$, where $Z\sim\mathcal N(0,I)$ and the cycling variance $q>0$ is fixed, then samples a clean endpoint using the exact stochastic reverse process. Forward and reverse randomness are independent across cycles.
\end{assumption}

The additive-noise form of Assumption~\ref{assump:exact-score} is a rescaling of the variance-preserving forward process used in our experiments. Noising to timestep $t$ gives $x_t=\sqrt{\bar\alpha_t}\,x+\sqrt{1-\bar\alpha_t}\,Z$, and the rescaled state $U=x_t/\sqrt{\bar\alpha_t}$ equals $x+\sqrt q\,Z$ with $q=(1-\bar\alpha_t)/\bar\alpha_t$. The cycling variance is therefore set by the cycling amplitude through $q(\gamma)=(1-\bar\alpha_{\gamma T})/\bar\alpha_{\gamma T}$, and it increases with $\gamma$. Throughout, residence refers to the clean cycle endpoints; intermediate noisy states and reverse trajectories need not remain near the target. 

\begin{theorem}[Recovery and cyclic residence]
\label{thm:recovery}
Under Assumption~\ref{assump:exact-score}, fix a training point $x_i$, let $d_{ij}=\|x_i-x_j\|$, and choose $0<R<\frac14\min_{j\ne i}d_{ij}$. For $0\le r\le R$, define
\[
A_i(r)=\min\left\{1,\frac12\sum_{j\ne i}
\exp\left[-\frac{d_{ij}(d_{ij}-4r)}{8q}\right]\right\}.
\]
Let $h_i(x)$ be the probability that one cycle starting at $x$ ends at $x_i$. If $\|x-x_i\|\le r$, then $h_i(x)\ge1-A_i(r)$. Consequently, for $X_0=x$ with $\|x-x_i\|\le R$ and any integer $M\ge1$,
\begin{equation}
\mathbb P_x(X_1=\cdots=X_M=x_i)
=h_i(x)h_i(x_i)^{M-1}
\ge[1-A_i(R)][1-A_i(0)]^{M-1}.
\label{eq:cyclic-persistence}
\end{equation}
\end{theorem}

\paragraph{Implications.} The theorem separates initial recovery from subsequent residence. Let $\delta_i=\min_{j\ne i}d_{ij}$. A cycle that starts within distance $R<\delta_i/4$ of $x_i$ returns to $x_i$ with high probability when the cycling variance $q$ is small relative to the separation margin $\delta_i(\delta_i-4R)$. Once recovered, each subsequent cycle starts at $x_i$ itself, and since $A_i(0)\le A_i(R)$, the return bound after recovery is at least as strong as the initial recovery bound. For any fixed number of cycles, the residence lower bound therefore approaches one as $q\to0$, which in the experiments corresponds to small cycling amplitude $\gamma$. The same bound shows how geometry controls residence. A larger neighbourhood radius or greater cycling noise weakens it, whereas greater separation from competing training points strengthens it. At the opposite extreme, as $q\to\infty$, the endpoint distribution from any fixed starting state approaches the uniform distribution over the $N$ training points, and the probability of $M$ consecutive endpoints at $x_i$ approaches $N^{-M}$. The theorem thus provides an idealized mechanism for persistent residence. Once a trajectory reaches a suitable neighbourhood, cycling repeatedly recovers the same training image. This mechanism has a defined scope. The analysis assumes an exactly memorized score, and it describes residence after a trajectory reaches the neighbourhood of a training point, not how the trajectory first reaches that neighbourhood from a distant initialization (see App.~\ref{sec:Theory} for background, the complete proof, and empirical checks).

\vspace{-0.1in}
\section{Experiments}
\label{sec:experiments}
\vspace{-0.1in}

In this section, we evaluate the geometric signatures and cyclic-denoising dynamics across training regimes defined relative to $\tau_{\mathrm{gen}}$ and $\tau_{\mathrm{mem}}$, on three datasets with several training configurations and on an off-the-shelf Stable Diffusion model, and establish the following. 
\vspace{-0.05in}
\begin{itemize}
\item \textbf{Geometric signatures of memorization} (\S~\ref{sec:exp-geometry}). Score divergence and basin volume separate training and generated samples from held-out samples before one-shot generation reproduces any training sample. The onset of this separation scales with $\tau/n$, as does $\tau_{\mathrm{mem}}$.
\item \textbf{Latent memorization} (\S~\ref{sec:exp-cyclic}). Within the generalization window, CD increases the memorized fraction from at most $0.1\%$ to over $30\%$. CD also reveals degenerate attractors, which are prevalent near $\tau_{\mathrm{gen}}$ and vanish as training proceeds.
\item \textbf{Stable Diffusion} (\S~\ref{sec:exp-sd}). Without retraining, CD recovers training images from initial generations that are not copies, and it increases the detection power of $\Delta\mathcal{A}_1$ at strict false-positive rates.
\end{itemize}

\vspace{-0.1in}
\subsection{Experimental Setup}
\vspace{-0.1in}
\textbf{Datasets.}
We train from scratch on three datasets (synthetic, CelebA, and CIFAR-10) and evaluate a pretrained model (Stable Diffusion) on a prompt collection. The synthetic dataset is an eight-dimensional, equally weighted GMM, $P_0=\tfrac12\mathcal{N}(\mathbf{1}_d,I_d)+\tfrac12\mathcal{N}(-\mathbf{1}_d,I_d)$, whose population score is available analytically, with training sets of size $n\in\{128,256,512,1024,2048,4096\}$. For CelebA~\citep{7410782}, we centre-crop and resize images to $32\times32$ grayscale and use subsets of size $n\in\{512,1024,2048,4096\}$. For CIFAR-10~\citep{krizhevsky2009learning}, we use images at their native $32\times32$ colour resolution and subsets of size $n\in\{2048,4096,8192\}$. For Stable Diffusion, we select memorized, partially memorized, non-memorized and plain english prompts from \citet{webster2023reproducibleextractiontrainingimages}, \citet{wenDetectingExplainingMitigating2024} and \citet{hongMemBenchMemorizedImage2024}. 

\textbf{Models and training.}
For the GMM, we train residual multilayer perceptrons (MLPs) with full-batch stochastic gradient descent (SGD). For CelebA, we train attention-based U-Nets with SGD or Adam. The GMM and CelebA models use a linear noise schedule from $\beta_1=10^{-4}$ to $\beta_T=2\times10^{-2}$ over $T=1{,}000$ diffusion timesteps. For CIFAR-10, we train unconditional pixel-space diffusion transformers (DiT-S/2)~\citep{Peebles2022ScalableDM} with Adam, using a cosine noise schedule \citep{nicholImprovedDenoisingDiffusion2021} with offset $s=0.008$ over $T=500$ diffusion timesteps. We continue training until $f_{\mathrm{mem}}$ is high, so that the saved checkpoints span pre-generalization, generalization, and memorization. We evaluate Stable Diffusion v1.4 off the shelf, without retraining or fine-tuning (see App.~\ref{app:finer-things-club} for full training and evaluation details).

\textbf{Cyclic-denoising protocol.}
We apply CD at three representative checkpoints: the checkpoint with the best generation quality before overfitting, a checkpoint after overfitting but before $\tau_{\mathrm{mem}}$, and a checkpoint after $\tau_{\mathrm{mem}}$. At each, we freeze the model and initialize trajectories either from its training images or from $n$ fresh DDIM generations, matching the training-set size. For a cycling amplitude $\gamma$, we set $t_\gamma=\operatorname{round}[\gamma(T-1)]$, noise each image to $t_\gamma$, and denoise it back to $t=0$ with the DDPM ancestral sampler, drawing fresh forward noise at every cycle. Each output initializes the next cycle, and the initial images form cycle zero. Trajectories are independent; batching serves only computation. For Stable Diffusion, we apply CD in latent space.

\textbf{Evaluation metrics.}
\label{sec:metrics-experiment}
\emph{Geometry.} At each checkpoint, we generate samples with DDIM \citep{songDenoisingDiffusionImplicit2020} and evaluate $\mathcal{A}_1$, $\widehat{\mathcal{A}}_1$, and basin volume on training, generated, and held-out samples (Section~\ref{sec:geom-sig}). We also report the training and held-out losses at the fixed low-noise level $\gamma=0.01$.

\emph{Memorization and sample quality.} Following
\citet{yoon2023diffusion,guMemorizationDiffusionModels2023},
we classify a generated sample as memorized when the ratio of its Euclidean distances to its nearest and second-nearest training
neighbours is below $\kappa=1/3$, and $f_{\mathrm{mem}}$ is the percentage of generated samples so classified. Error bars denote one bootstrap standard deviation of the mean over $1{,}000$ resamples of the generated set. Sample quality is measured by the Fréchet inception distance ($\mathbf{FID}$) \citep{Heusel2017GANsTB} against disjoint held-out sets for CelebA and CIFAR-10, and by a Monte Carlo estimate of $\boldsymbol{D_{\mathrm{KL}}(P_\theta\|P_0)}$ from $10{,}000$ generated samples for the GMM.

\emph{Cyclic dynamics.}
We track $S_{\mathrm{cyc}}^{(k)}$, $\mathcal{A}_1$, $\widehat{\mathcal{A}}_1$, and $f_{\mathrm{mem}}$ as functions of the cycle index $k$ and the cycling amplitude $\gamma$ (\S~\ref{sec:cyclic-probe}).

\emph{Stable Diffusion.}
We use two copy-detection quantities built on the $\ell_2$-normalized SSCD descriptor $\phi$ \citep{Pizzi_2022_CVPR}.
The similarity to training data, $\boldsymbol{\mathrm{SSCD}(x,p)}=\max_{g\in G(p)}\langle\phi(x),\phi(g)\rangle$, compares a sample with the set $G(p)$ of training images matched to prompt $p$.
The seed agreement, $\boldsymbol{\mathrm{SA}(p)}=\binom{S}{2}^{-1}\sum_{i<j}\langle\phi(x^{(i)}),\phi(x^{(j)})\rangle$, is the mean pairwise similarity of the $S=4$ generations of $p$ under different seeds and measures basin collapse without reference to any training image. We also track the conditional--unconditional gap $\boldsymbol{\Delta\mathcal{A}_1}=\mathcal{A}_1(x;c_p)-\mathcal{A}_1(x;\epsilon)$, where $c_p$ is the CLIP text embedding of $p$ and $\phi$ that of the empty prompt, which forms the unconditional branch of classifier-free guidance.
Subtracting the unconditional branch removes the curvature the model has at $x$ independently of any prompt (see App.~\ref{app:metrics} for definitions and implementations of all metrics).

\vspace{-0.1in}
\subsection{Evolution of Geometric Signatures Across Training}
\vspace{-0.1in}
\label{sec:exp-geometry}

Figure~\ref{fig:training-dynamics} shows the losses, the mean normalized score divergence $\widehat{\mathcal{A}}_1$ on training, held-out, and generated samples, the mean log basin volume, and $f_{\mathrm{mem}}$ across training on CelebA at $n=1024$. 

\textit{Divergence separates before one-shot memorization.} The held-out loss reaches its minimum at $\tau=45{,}000$ and then rises while the training loss continues to fall, which marks the overfitting transition (Fig.~\ref{fig:training-dynamics}A). The divergence of training and held-out samples coincides until $\tau\approx45{,}000$ and separates at $\tau_{\mathrm{div}}=65{,}000$ (Fig.~\ref{fig:training-dynamics}B). The first non-zero $f_{\mathrm{mem}}$ appears only at $\tau_{\mathrm{mem}}=90{,}000$, $25{,}000$ steps after $\tau_{\mathrm{div}}$, and $f_{\mathrm{mem}}$ exceeds $1\%$ only at $\tau=155{,}000$. Between $\tau_{\mathrm{div}}$ and $\tau_{\mathrm{mem}}$, the model is in the latent-memorization regime: the score field contracts around training samples, yet one-shot generation reproduces none of them.

\textit{Dependence on noise level and dataset size.}
The separation appears at small noise levels. Across the $\gamma$ sweep at every checkpoint (Fig.~\ref{fig:regime-gamma-grid}), the divergence stays in a narrow band near zero in the generalized regime, whereas in the memorized regime it is strongly negative at small $\gamma$ and returns toward zero as $\gamma$ grows. Small $\gamma$ resolves structure on the scale of individual basins and large $\gamma$ averages over many, so we report $\gamma=0.001$. The timing depends on the dataset size. Both the separation and the rise in $f_{\mathrm{mem}}$ occur later for larger $n$, and replotting against $\tau/n$ collapses the onsets for $n\in\{512,1024,2048\}$ (Fig.~\ref{fig:training-dynamics}B, inset). The geometric signature thus follows the $O(n)$ scaling established for $\tau_{\mathrm{mem}}$~\citep{bonnaireWhyDiffusionModels2025}.

\textit{Basin volume and generated samples.}
Basin volume separates in the same direction (Fig.~\ref{fig:training-dynamics}C). Its mean grows for training and generated samples and decreases for held-out samples. Its standard deviation decreases for training and generated samples and stays roughly constant for held-out samples, which indicates that basins form increasingly uniformly around training examples. From $\tau=5{,}000$ onward, the divergence of generated samples follows that of training samples with a mean absolute difference of $0.015$, and the two distributions are almost identical across training configurations (App.~\ref{app:additional-results}). Generated samples can therefore replace the training set in the contrast with held-out samples, which makes the signature readable from the model and held-out reference data alone. Appendix~\ref{app:membership} uses this property for membership inference from score divergence.
\begin{figure*}[ht]
\vspace{-0.2in}
    \centering
    \includegraphics[width=\textwidth]{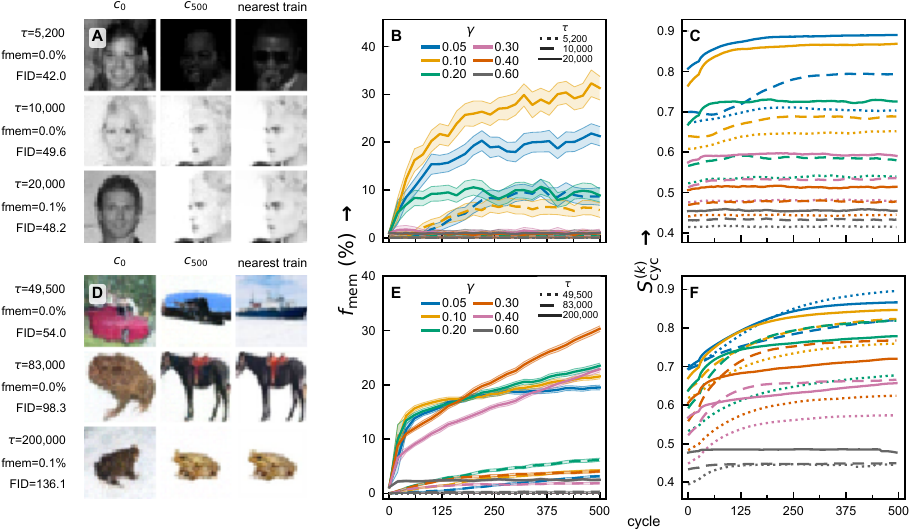}
\vspace{-0.2in}
\caption{
\textbf{Cyclic denoising exposes latent memorization inside the generalization window.}
CelebA ($n=4096$, Adam; \textbf{A--C}) and CIFAR-10 ($n=8192$; \textbf{D--F}) at checkpoints from $\tau_{\mathrm{gen}}$ to $\tau_{\mathrm{mem}}$.
(\textbf{A,D}) Generated samples at $c_0$ and $c_{500}$ ($\gamma=0.10$) with their nearest training images.
(\textbf{B,E}) $f_{\mathrm{mem}}$ and (\textbf{C,F}) $S_{\mathrm{cyc}}^{(k)}$ across cycles for several $\gamma$.}
\vspace{-0.2in}
\label{fig:cyclic-population}
\end{figure*}

\vspace{-0.1in}
\subsection{Cyclic Denoising Dynamics}
\label{sec:exp-cyclic}
\vspace{-0.1in}
We now test whether the basins that form in the latent-memorization regime are reachable through CD, on CelebA ($n=4096$, Adam) and CIFAR-10 ($n=8192$) (Fig.~\ref{fig:cyclic-population}).

\textit{CD recovers training images that one-shot generation misses.}
At CelebA $\tau=10{,}000$ and CIFAR-10 $\tau=83{,}000$, one-shot samples exhibit $f_{\mathrm{mem}}=0\%$, yet $f_{\mathrm{mem}}$ rises with cycling (Fig.~\ref{fig:cyclic-population}B,E). At CelebA $\tau=20{,}000$, where one-shot $f_{\mathrm{mem}}$ is $0.1\%$, cycling at $\gamma=0.10$ raises $f_{\mathrm{mem}}$ above $30\%$ by cycle 500. Continued exploration brings more trajectories into basins of individual training samples, although only a subset of training samples form such basins.

\textit{The outcome depends on the cycling amplitude and the training stage.}
In several CelebA settings, $f_{\mathrm{mem}}$ plateaus within 500 cycles (Fig.~\ref{fig:cyclic-population}B). At that $\gamma$, few additional trajectories find detectable training-sample basins within the observed horizon, although exploration need not have stopped. At CIFAR-10 $\tau=200{,}000$ and $\gamma=0.30$, $f_{\mathrm{mem}}$ instead keeps rising through cycle 500 (Fig.~\ref{fig:cyclic-population}E). At CelebA $\gamma=0.40$, neither $f_{\mathrm{mem}}$ nor $S_{\mathrm{cyc}}^{(k)}$ rises appreciably, even at $\tau=20{,}000$ (Fig.~\ref{fig:cyclic-population}B,C). Perturbations at this amplitude exceed the range over which the accessible basins remain stable, and no residence occurs. This trend matches Thm.~\ref{thm:recovery}, where larger cycling variance weakens the residence bound.

\textit{Geometry changes before identification.}
Individual trajectories follow the same order of events observed across training (App.~Fig.~\ref{fig:attractor-grid}). Their images approach particular training examples over hundreds of cycles (A,D). The drop in $\mathcal{A}_1$ often begins before the trajectory is persistently classified as memorized (B,E), and $S_{\mathrm{cyc}}^{(k)}$ then reaches a high plateau (C,F), consistent with continued residence. Score divergence also decreases across cycles at the population level (App.~\ref{app:additional-results}). Within a trajectory, as within training, the geometry changes first and identification with a training sample follows.

\textit{Residence alone does not establish memorization.}
At the earlier $\tau_{gen}$ CelebA and CIFAR-10 checkpoints, $S_{\mathrm{cyc}}^{(k)}$ approaches $0.9$ at some amplitudes while $f_{\mathrm{mem}}$ stays near zero (Fig.~\ref{fig:cyclic-population}B,C,E,F). These trajectories reside in degenerate attractors. On CIFAR-10 at $\gamma=0.05$, degenerate attractors hold $50.9\%$ of trajectories at $\tau_{\mathrm{gen}}=49{,}500$, $16.0\%$ at the generalization-window checkpoint $\tau=83{,}000$, and $0.98\%$ by $\tau=200{,}000$. Their appearance changes with training (Fig.~\ref{fig:degenerate-attractor-grid}). The $\tau_{\mathrm{gen}}$ states include stripes, nearly monochromatic images, and mixed-colour patterns, whereas some states at $\tau=83{,}000$ contain recognizable objects while matching no single training image. In both cases, $S_{\mathrm{cyc}}^{(k)}$ rises above $0.9$ and saturates, often alongside a decrease in $\mathcal{A}_1$. \citet{sharmaCyclicDenoisingReveals2026} also observe such regions in the landscape and refer to these images as \emph{trivial points}. Residence in these attractors is sensitive to the cycling amplitude: at $\tau_{\mathrm{gen}}$, no trajectory passes the residence screen at $\gamma=0.20$ or $0.30$. Appendix~\ref{app:additional-results} reports further results on the same.

These results separate two properties that CD measures. $S_{\mathrm{cyc}}^{(k)}$ and $\mathcal{A}_1$ detect residence in a basin and require only the model, whereas identifying that basin with a training sample requires the $f_{\mathrm{mem}}$ test and the training set. Inside the generalization window, the share of degenerate attractors falls while $f_{\mathrm{mem}}$ under cycling rises, so residence increasingly occurs in training-sample basins that one-shot generation does not reach. How degenerate basins arise and lose accessibility as training proceeds remains an open question.

\begin{figure*}[t]
\vspace{-0.2in}
\centering
\includegraphics[width=\linewidth]{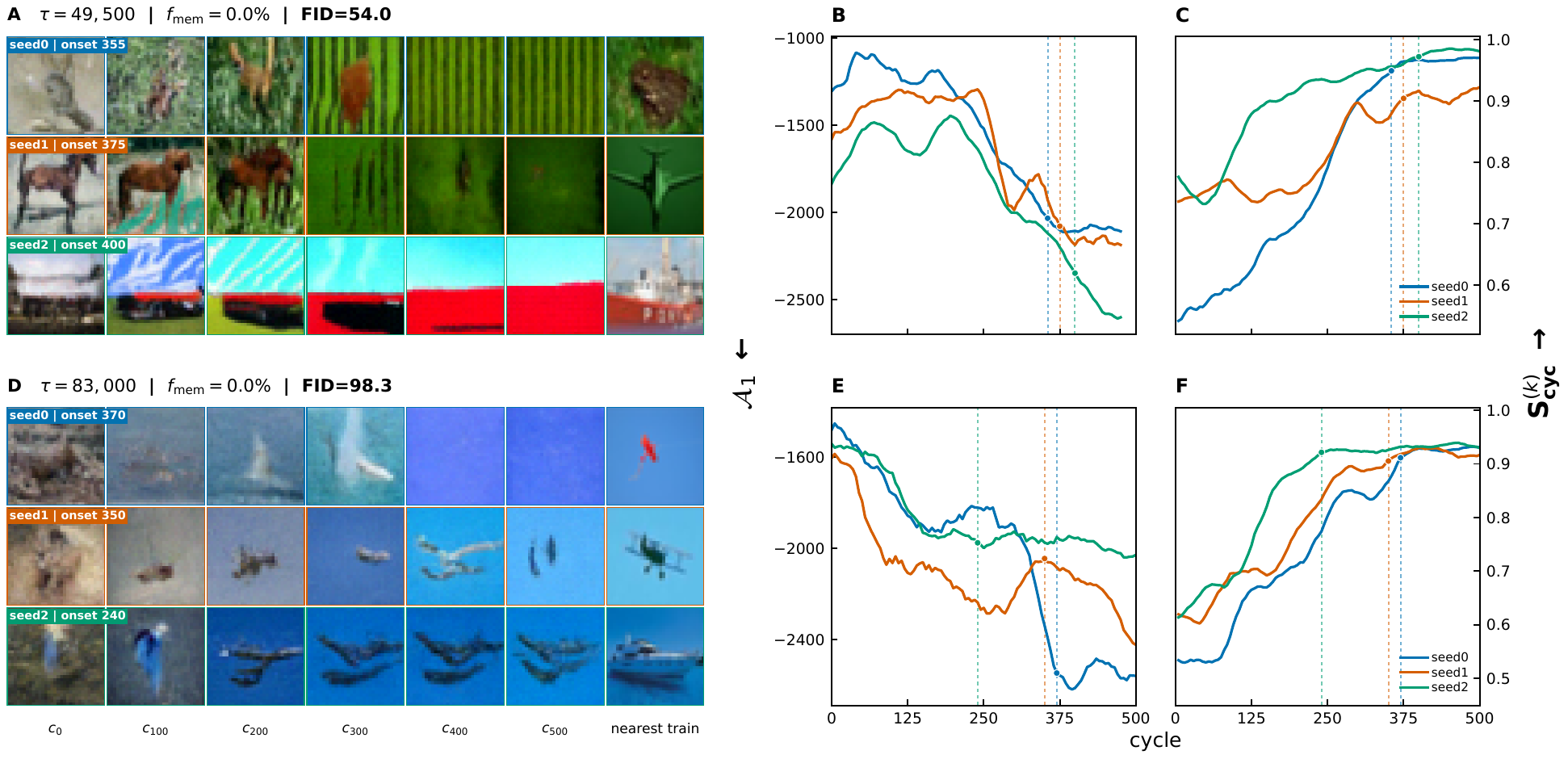}
\vspace{-0.2in}
\caption{
\textbf{Cyclic denoising reveals degenerate attractors.}
CIFAR-10 ($n=8192$, $\gamma=0.05$) at $\tau_{\mathrm{gen}}=49{,}500$ (\textbf{A--C}) and $\tau=83{,}000$ (\textbf{D--F}). (\textbf{A,D}) Three trajectories from $c_0$ to $c_{500}$ with the nearest training image to the final state; no attractor matches a single training image. (\textbf{B,E}) $\mathcal{A}_1$ and (\textbf{C,F}) $S_{\mathrm{cyc}}^{(k)}$ across cycles. Markers show the residence onset, which does not indicate memorization.}
\vspace{-0.2in}
\label{fig:degenerate-attractor-grid}
\end{figure*}

\vspace{-0.1in}
\subsection{Experiments on Stable Diffusion}
\label{sec:exp-sd}
\vspace{-0.05in}

We next test whether these observations extend to off-the-shelf Stable Diffusion v1.4, using prompts from \citet{webster2023reproducibleextractiontrainingimages}, \citet{wenDetectingExplainingMitigating2024} and  \citet{hongMemBenchMemorizedImage2024}.

\textit{Individual chains.}
The six chains in App.~Fig.~\ref{fig:sd}A each isolate one factor. The \texttt{hippo} chain reaches its training image at $c_1$ and stays there through $c_{200}$. The two \texttt{laura} chains share a prompt, a seed, and a $c_0$ image with SSCD $0.203$ to the training image; the chain at $\gamma=0.1$ never arrives, whereas the chain at $\gamma=0.9$ settles from $c_{35}$. The two \texttt{air-con} chains share a prompt and $\gamma$ and differ only in seed: one enters its basin early, leaves it near $c_{100}$, and returns, whereas the other arrives only at $c_{50}$. The \texttt{label tape} chain arrives at $c_{70}$ from a $c_0$ image that bears no resemblance to its training image. Arrival thus depends on the cycling amplitude and the seed, can occur long after cycling begins, can be transient, and does not require the initial image to resemble the training image. The \texttt{laura} and \texttt{label tape} chains show latent memorization in a large-scale model: the initial generation is not a copy, and cycling recovers the training image. In all six chains, $\mathrm{SSCD}(x,p)$ and $\Delta\mathcal{A}_1$ change together at arrival and departure (App.~Fig.~\ref{fig:sd}B,C).

\textit{Population behaviour.}
Memorized prompts show rising $\mathrm{SA}(p)$ and increasingly negative $\Delta\mathcal{A}_1$ under cycling, whereas non-memorized prompts change little in either (App.~Fig.~\ref{fig:sd}E). After 200 cycles, $\Delta\mathcal{A}_1$ is graded by the degree of memorization and stays near zero for non-memorized and plain-English controls (App.~Fig.~\ref{fig:sd}D).

\textit{Detection.}
On 500 memorized and 2,000 non-memorized prompts from \citep{wenDetectingExplainingMitigating2024}, cycling at $\gamma=0.9$ raises the AUC of $\Delta\mathcal{A}_1$ from $0.811$ at $c_0$ to $0.944$ at $c_{60}$, and its TPR at $1\%$ FPR from $0.526$ to $0.866$ (App.~\ref{app:SD_AUC}). The conditional--unconditional subtraction matters at strict false-positive rates: the conditional divergence alone reaches a similar AUC of $0.948$ but a TPR at $1\%$ FPR of only $0.194$. On the same prompts and seeds, the single-shot detector of \citet{wenDetectingExplainingMitigating2024} achieves a higher AUC of $0.990$ and a comparable TPR at $1\%$ FPR of $0.838$. The comparison depends on the negative source. Against LAION negatives, which match the training distribution, $\Delta\mathcal{A}_1$ reaches a TPR at $1\%$ FPR of $0.820$ against $0.740$ for the single-shot detector, whereas the single-shot detector remains stronger against COCO, Lexica, and random-token negatives (App.~\ref{app:SD_AUC}). CD thus transfers to a model trained at scale without retraining. It reveals when and how a trajectory reaches a memorized image, and it adds detection power at strict false-positive rates on in-distribution negatives.

\vspace{-0.1in}
\section{Conclusion}
\vspace{-0.1in}
We studied how memorization develops in diffusion models through the geometry of the learned energy landscape. Localized basins form around individual training examples before one-shot generation produces copies. Score divergence and basin volume separate training from held-out samples in this latent-memorization regime, and their onset follows the same $O(n)$ scaling as $\tau_{\mathrm{mem}}$. Cyclic denoising reaches the training-sample basins that one-shot sampling misses and recovers training images from models whose one-shot samples contain none. Residence under cycling alone does not imply memorization: degenerate attractors are prevalent near $\tau_{\mathrm{gen}}$ and fade as training proceeds. Under the exact empirical score, a well-separated training point retains cycling trajectories with a probability controlled by the cycling noise and its separation from competing points. These behaviours hold across a Gaussian mixture, CelebA, CIFAR-10, and off-the-shelf Stable Diffusion v1.4. Memorization is encoded in the landscape geometry before it is expressed in generated samples, and probing this geometry offers a route to auditing models whose outputs appear free of copies. Limitations and future works arising from the present work are discussed in App.~\ref{app:limitations}.
\newpage

\subsection*{AI Use Statement}
All AI-assisted suggestions, analyses, and manuscript changes were reviewed by the authors. The authors independently verified the reported experimental results and mathematical claims and take full responsibility for the final technical content, analyses, and conclusions of the paper. We used the Stanford Agentic Reviewer (https://paperreview.ai) as author-side review tools to obtain preliminary feedback and identify points requiring further clarification or validation.

\subsection*{Ethics Statement}
This work studies the emergence of memorization in diffusion models
through the geometry of the learned energy landscape and the dynamics of
cyclic denoising. Understanding these dynamics can inform future work on
privacy auditing, training data protection, and mitigating unintended
memorization in generative models. Our experiments use publicly
available datasets and models in controlled research settings. The analysis is 
intended to improve understanding of memorization, thereby supporting the development of safer and
more privacy-preserving generative models.
 
\subsection*{Reproducibility Statement}
The code will be made publicly available upon acceptance of the paper. The appendix details the full experimental protocol, including hyperparameters, model and training details, sampling configurations, data preprocessing and splits, and metric definitions and computation, together with the implementation details needed to reproduce the principal results.

\clearpage
\appendix

\etocdepthtag.toc{appendix}

\begingroup
\etocsettagdepth{main}{none}
\etocsettagdepth{appendix}{subsection}
\etocsettocdepth{subsection}

\etocsettocstyle{%
\section*{Appendix Contents}
\label{app:contents}

}{}

\setlength{\parskip}{0pt}
\tableofcontents
\endgroup

\section{Notation}
\label{app:notation}

Table~\ref{tab:notation} summarizes the notation used throughout the paper.

\begin{table*}[t]
\centering
\caption{Notation used throughout the paper.}
\label{tab:notation}

\scriptsize
\renewcommand{\arraystretch}{0.92}
\setlength{\tabcolsep}{3pt}

\begin{tabular}{
    p{0.09\textwidth} p{0.37\textwidth}
    p{0.09\textwidth} p{0.37\textwidth}
}
\toprule
\textbf{Notation} & \textbf{Meaning} &
\textbf{Notation} & \textbf{Meaning} \\
\midrule

\multicolumn{4}{l}{\textit{\textbf{Diffusion model}}} \\[-1pt]

$x_0$ &
Clean data sample. &
$x_t$ &
Noisy sample at diffusion timestep $t$. \\

$t$ &
Diffusion timestep. &
$T$ &
Total number of diffusion timesteps. \\

$\beta_t$ &
Forward-process variance schedule. &
$\alpha_t=1-\beta_t$ &
Signal-retention coefficient. \\

$\bar{\alpha}_t=\prod_{s=1}^{t}\alpha_s$ &
Cumulative signal-retention coefficient. &
$\epsilon\sim\mathcal{N}(0,I)$ &
Gaussian noise in the forward process. \\

$\epsilon_\theta(x_t,t)$ &
Network prediction of injected noise. &
$q(x_t\mid x_0)$ &
Forward diffusion distribution. \\

$p_\theta(x_{t-1}\mid x_t)$ &
Learned reverse transition. &
$p_t(x)$ &
Data density at diffusion time $t$. \\

$s_t(x)=\nabla_x\log p_t(x)$ &
Exact score function. &
$s_\theta(x,t)$ &
Learned score approximation. \\

$E_t(x)=-\log p_t(x)$ &
Time-dependent energy. &
$d$ &
Dimensionality of the data space. \\

\midrule
\midrule
\multicolumn{4}{l}{\textit{\textbf{Geometric quantities}}} \\[-1pt]
$\mathcal A(x,t)$ &
Score divergence: reported unscaled mean outward
learned-score flux. &
$\widehat{s}_\theta(x,t)
=s_\theta(x,t)/\|s_\theta(x,t)\|_2$ &
Unit-normalized learned score, where defined. \\
$\widehat{\mathcal A}(x,t)$ &
Normalized score divergence: mean outward alignment
of unit-normalized boundary scores. &
$R_{\mathrm p}$ &
Physical radius of the score-probing sphere. \\
$e_j\sim\mathrm{Unif}(\mathbb S^{d-1})$ &
Random outward unit probe direction. &
$J$ &
Number of boundary probes ($J=64$). \\
$V_{\mathrm{basin}}$ &
Effective basin/recovery volume. &
$R$ &
Radius of the theoretical neighborhood $B_R$. \\
\midrule
\multicolumn{4}{l}{\textit{\textbf{Training and memorization}}} \\[-1pt]

$\tau$ &
Training/optimization step. &
$\tau_{\mathrm{gen}}$ &
Onset of the generalization regime. \\

$\tau_{\mathrm{mem}}$ &
Onset of visible memorization. &
$f_{\mathrm{mem}}$ &
Fraction of generated samples classified as memorized. \\

$n$ &
Training-set size. &
$\kappa$ &
Threshold in the nearest-neighbor memorization criterion. \\

\midrule
\multicolumn{4}{l}{\textit{\textbf{Cyclic denoising}}} \\[-1pt]

$\gamma=t/T$ &
Normalized cyclic-denoising noise amplitude. &
$x^{(k)}$ &
Sample after cycle $k$. \\

$U_t(x'\mid x)$ &
One-cycle transition kernel. &
$k$ or $c$ &
Cyclic-denoising iteration index. \\

\midrule
\multicolumn{4}{l}{\textit{\textbf{Theoretical analysis}}} \\[-1pt]

$x_i,x_j$ &
Individual training samples. &
$N$ &
Number of samples in the empirical distribution. \\

$q$ &
Noise variance in $U=x_i+\sqrt{q}Z$. &
$Z\sim\mathcal{N}(0,I)$ &
Gaussian noise in the theoretical analysis. \\

$U$ &
Noisy observation of a training sample. &
$w_i(u)$ &
Posterior probability associated with $x_i$. \\

$d_{ij}=\|x_j-x_i\|$ &
Distance between training samples $x_i,x_j$. &
$B_R$ &
Radius-$R$ neighborhood around $x_i$. \\

$A_i(r)$ &
Upper-bound term for failure to recover $x_i$. &
$h_i(x)$ &
Probability that one cycle returns to $x_i$. \\

$M$ &
Number of cycles in the persistence bound. &
$r$ &
Distance/radius around training sample $x_i$. \\

\midrule
\multicolumn{4}{l}{\textit{\textbf{Evaluation metrics}}} \\[-1pt]

$\mathrm{FID}$ &
Fr\'echet Inception Distance. &
$D_{\mathrm{KL}}(P_\theta\|P_0)$ &
KL divergence used for GMM evaluation. \\

$\mathrm{SSCD}$ &
Copy-detection similarity. &
$\Delta\hat{A}$ &
Conditional--unconditional geometric-statistic difference. \\

\bottomrule
\end{tabular}
\end{table*}

\section{Extended Related Works}
\label{app:extended_rw}

\paragraph{Memorisation in Diffusion Models.}
\citet{bonnaireWhyDiffusionModels2025} distinguish a generalization time $\tau_{\mathrm{gen}}$, when high-quality novel samples first appear, from a later memorization time $\tau_{\mathrm{mem}}$, when generated samples begin reproducing the training set. They find that both loss-level overfitting and visible memorization occur on $O(n)$ training timescales, where $n$ is the number of training samples. These events do not coincide: training and test losses separate while the memorization fraction $f_{\mathrm{mem}}$ remains near zero. \citet{faveroBiggerIsntAlways2025} report similar memorization-time scaling across image and language diffusion models and describe generalization and memorization as competing training timescales. Both works motivate early stopping.

Connecting diffusion models with dense associative memories, \citet{phamMemorizationGeneralizationEmergence2025} vary dataset size while holding the model and training procedure fixed. Small datasets yield attractors associated with individual training examples; intermediate sizes yield spurious attractors; and larger datasets yield generalized states on flatter manifolds. Their geometric measurements distinguish these regimes: memorized states have the largest measured basin volumes, spurious states have smaller but nonzero basins, and generalized states have nearly vanishing basin volume. Their comparison varies the data regime, whereas our analysis follows basin geometry along the optimization axis of a fixed model and dataset.

\paragraph{Geometric detection of memorization.}
Several methods detect memorization through the geometry of trained diffusion models. For conditional text-to-image models, \citet{wenDetectingExplainingMitigating2024} use the norm of the difference between conditional and unconditional scores to identify memorized prompts, including at the first denoising step. Building on this measure, \citet{jeonUnderstandingMitigatingMemorization2025} relate the score difference to sharpness of the log-probability density and introduce a Hessian-based measure of sharpness at the start of generation. \citet{asthanaDetectingMitigatingMemorization2026} show that the low-noise landscape becomes anisotropic, so a scalar score-difference norm can miss directional structure. They combine the norm with the angular alignment between the classifier-free guidance vector and the unconditional score. \citet{brokmanIdentifyingMemorizationDiffusion2025,zotero-item-106} construct Monte Carlo estimators of the $p$-Laplacian from the learned score to identify memorized samples. These methods evaluate trained models; our analysis asks how geometric signatures develop during training.

\paragraph{Cyclic Denoising.}
Geometric measurements describe the local structure of the learned landscape, but not how trajectories move within or between its regions.
Repeated partial noising and denoising probes this dynamical structure directly.
\citet{kangSamplingDataChains2026} iterate forward--backward diffusion steps to form a U-Turn Markov chain (UTMC) and study its dynamics across cycles at fixed U-Turn noise levels.
In their synthetic hierarchical model, low-noise UTMCs restrict exploration and cannot connect disconnected components of the data manifold, leading to non-ergodic dynamics.
Larger noise levels enable transitions between components and restore ergodicity.
At low noise, high-level features retain memory of the initial sample longer than low-level features; this ordering reverses at larger noise levels.
They observe the same qualitative behavior in natural-language and image experiments.
\citet{sharmaCyclicDenoisingReveals2026} study repeated forward--reverse steps as cyclic denoising (CD), using it as an unconditional extraction attack and a probe of the learned landscape.
At low corruption, trajectories can enter trivial fixed points or short periodic cycles.
At larger corruption, they rearrange, cross between basins, and become trapped in long-lived structured attractors, many of which reproduce training images.
They use high similarity between successive cycle outputs as an indication of basin residence and uncover images close to potential training examples.
They demonstrate this behavior and the resulting extraction on Stable Diffusion v1.4 and on a denoising diffusion probabilistic model (DDPM) trained on CIFAR.
Together, these works offer distinct views of CD: \citet{kangSamplingDataChains2026} study landscape exploration and derive theoretical results on ergodicity and feature relaxation in a synthetic hierarchical model, while \citet{sharmaCyclicDenoisingReveals2026} demonstrate its utility in training data extraction from fixed trained models. What remains missing is an account of how CD behavior changes as memorization develops during training, when it can extract memorized images, and the dynamics underlying its behavior. We investigate these questions across training and provide idealized theoretical analysis for CD.

\section{Limitations and future works}
\label{app:limitations}

Some of the major limitations of the work are discussed below. These could be addressed as part of future works.
\begin{itemize}
    \item Our study of how memorization develops during training covers small datasets at $32\times32$ resolution with models trained from scratch. Stable Diffusion, although studied, is only evaluated as a single released checkpoint. Whether latent memorization forms on the same timescales, and remains detectable, in large conditional models trained on web-scale data is untested. 
    \item Similarly, residence under cyclic denoising identifies basins but not their origin. Linking a basin to a training sample still requires the training set or a set of candidate images, and degenerate attractors show that residence alone can be misleading.
    \item The geometric probes carry their own constraints. In high dimensions, score divergence is estimated from finite boundary samples at a chosen probe scale and noise level, and the signal fades at larger noise. The estimates therefore depend on these choices. 
    \item CD needs tens to hundreds of cycles per trajectory, which makes it far more expensive than single-shot detectors. Its benefit on Stable Diffusion is limited to strict false-positive rates on in-distribution negatives.
    \item The theory assumes the exact empirical score, the fully memorized limit. It explains why trajectories stay near a training point once they arrive, but not how they arrive from distant starts, and not how residence behaves under the learned score in the latent regime.
    \item Finally, the same probe that audits a model can extract training data from it, as our Stable Diffusion experiments show. Deploying cyclic denoising for auditing therefore requires the access controls discussed in Appendix~\ref{app:broader-impact}.
\end{itemize}

\section{Broader Impact}
\label{app:broader-impact}
Understanding how memorization emerges in diffusion models has implications for both trustworthy generative modeling and responsible model deployment. Our results suggest that memorization can leave detectable geometric and dynamical signatures before it becomes obvious through ordinary generation, creating opportunities for earlier auditing of models for privacy leakage, training-data reproduction, and copyright-sensitive behavior. The proposed landscape-based diagnostics and cyclic-denoising probes may therefore help practitioners compare training configurations, identify models or checkpoints that are beginning to encode sample-specific memories, and motivate memorization-aware stopping or mitigation strategies. At the same time, because cyclic denoising can expose training-associated samples that standard generation may not readily reveal, the same techniques could potentially be used for model extraction or privacy attacks. We therefore view these methods primarily as diagnostic tools for understanding and auditing diffusion models, and encourage their use alongside appropriate access controls and responsible evaluation practices. The broader message is that model reliability should be assessed not only from generated outputs, but also from the hidden geometric structure and dynamical stability of memories encoded in the learned distribution.

\section{Some Preliminaries}
\subsection{Diffusion Models}
\label{app:prelim-diff}
Diffusion models (DMs) are generative models that learn to transform a simple noise distribution into a complex data distribution by reversing a fixed, progressive corruption process \citep{songDenoisingDiffusionImplicit2020}. Given a clean data sample $x_0 \sim p_{\text{data}}$ and a predefined variance schedule $\beta_1, \dots, \beta_T$, the forward process injects Gaussian noise
over $T$ discrete timesteps. Defining $\alpha_t = 1 - \beta_t$ and $\bar{\alpha}_t = \prod_{s=1}^t \alpha_s$, the forward marginal is
\begin{equation}
    Q(x_t \mid x_0)
    = \mathcal{N}\!\left(x_t;\, \sqrt{\bar{\alpha}_t}\, x_0,\,
      (1 - \bar{\alpha}_t)I\right).
\end{equation}
This lets a noisy sample at any timestep $t$ be written in closed form as
\begin{equation}
    x_t = \sqrt{\bar{\alpha}_t}\, x_0 + \sqrt{1 - \bar{\alpha}_t}\,\epsilon,
    \qquad \epsilon \sim \mathcal{N}(0, I).
\end{equation}

To generate new samples, DMs learn a reverse process that gradually denoises a sample initialized from the prior $x_T \sim \mathcal{N}(0, I)$. A network $\epsilon_\theta$ with parameters $\theta$ is trained to predict the injected noise via the simplified denoising objective \citep{hoDenoisingDiffusionProbabilistic2020}:
\begin{equation}
    \mathcal{L}(\theta)
    = \mathbb{E}_{x_0, \epsilon, t}
      \left[\, \| \epsilon - \epsilon_\theta(x_t, t) \|^2 \,\right].
\end{equation}
This objective is mathematically equivalent to denoising score matching \citep{songScoreBasedGenerativeModeling2020}: the score of the perturbed distribution, $s_t(x_t) = \nabla_{x_t} \log p_t(x_t)$, is recovered directly from the model's noise prediction,
\begin{equation}
    s_\theta(x_t, t)
    = -\frac{\epsilon_\theta(x_t, t)}{\sqrt{1 - \bar{\alpha}_t}}.
    \label{eq:score_estimate}
\end{equation}
Sampling from a diffusion model can therefore be understood as simulating a probability-flow or Langevin dynamics governed by the learned score field $s_\theta(x_t, t)$.

\subsection{\texorpdfstring{$\gamma$ to $\sigma_t$}{gamma to sigma\_t}}
\label{app:sigma-gamma}

\begin{figure}[ht]
    \centering
    \includegraphics[width=0.8\linewidth]{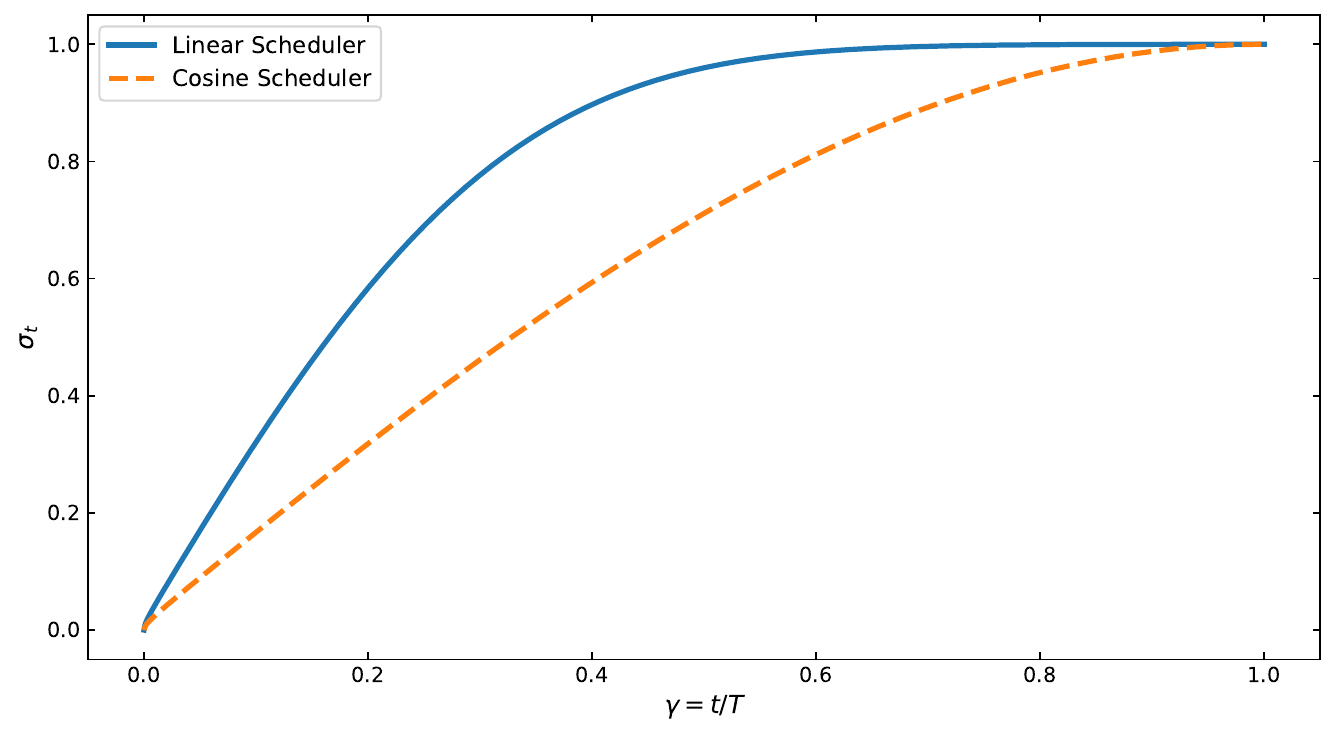}
\caption{Comparison of the noise standard deviation ($\sigma_t = \sqrt{1 - \bar{\alpha}_t}$) versus the normalized timestep ($\gamma = t/T$) for linear and cosine noise schedulers. The cosine scheduler maintains lower noise magnitude for a significantly broader range of timesteps.}
    \label{fig:sigma_t_vs_gamma_comparison}
\end{figure}
\FloatBarrier

\section{Theoretical Insight}
\label{sec:Theory}

\subsection{Background}
\label{sec:Theory_bg}

\paragraph{Why should cyclic denoising reveal memorized samples?}
Our experiments suggest that, as memorization develops, individual training
examples become increasingly stable under repeated noise--denoise cycles.
This raises a natural theoretical question: \emph{when does a training sample
exhibit prolonged residence under cyclic denoising, and how does this
residence depend on the noise level and separation from other training
examples?}
To isolate this mechanism, we analyze the exact empirical-score
limit~\citep{biroliDynamicalRegimesDiffusion2024}, in which the clean
distribution is the empirical training distribution and denoising performs
exact stochastic reversal.

\paragraph{Exact empirical-score limit.}
Let $\mathcal D=\{x_1,\ldots,x_N\}\subset\mathbb R^d$ consist of distinct
training points, with $N\ge2$, and let
\[
    p_0=\frac1N\sum_{j=1}^N\delta_{x_j}.
\]
Fix a cycling noise variance $q>0$. In additive-noise coordinates, a cycle
starting from $x$ first produces
\[
    U=x+\sqrt q\,Z,
    \qquad Z\sim\mathcal N(0,I).
\]
We assume that denoising uses the exact empirical score and samples from
the corresponding exact reverse endpoint distribution. Bayes' rule gives
\[
    \mathbb P(X'=x_i\mid U=u)=w_i(u),
    \qquad
    w_i(u)=
    \frac{\exp[-\|u-x_i\|^2/(2q)]}
         {\sum_{j=1}^N\exp[-\|u-x_j\|^2/(2q)]}.
\]
Each cycle uses fresh independent forward noise and reverse-sampling
randomness. Starting from any initial state, the first complete cycle
produces a training point; subsequent clean endpoints form a
time-homogeneous Markov chain on $\mathcal D$.

\paragraph{Isolation controls recovery and residence.}
Recovery depends on the noisy observation's relative compatibility with
the target and all competing training points. The relevant quantities
are the inter-sample distances $d_{ij}=\|x_i-x_j\|$, the cycling variance
$q$, and the initial distance from the target. Starting sufficiently
close to $x_i$, small noise relative to its separation from competitors
makes recovery reliable. The theorem quantifies this through a sum over
all competitors, accounting for both their distances and their number.
Increasing the noise weakens the resulting recovery and residence bounds.

\paragraph{What the theorem establishes.}
For an initialization within a sufficiently small neighborhood of $x_i$,
the theorem lower-bounds the probability of recovering $x_i$ in one
complete cycle. Conditional on recovery, subsequent cycles start at
$x_i$ itself, allowing a bound on repeated return to the same training
point. For a fixed dataset and a neighborhood satisfying the theorem's
separation condition, sufficiently small noise yields high-probability
residence over any prescribed finite horizon. Residence concerns clean
endpoints; intermediate noisy states and reverse trajectories need not
remain in the neighborhood. The theorem does not establish permanent
trapping or discovery of the neighborhood from distant initializations.
The empirical validation in this appendix evaluates the
theorem's quantitative prediction and lower bound using
exact empirical posterior sampling on a GMM dataset, then
examines approximate recovery and residence near training
images using fully memorized and low-memorization checkpoints
of a U-Net diffusion model trained on CelebA. The learned-model evaluation assesses
the practical relevance of the mechanism using a reconstruction-distance
tolerance; the exact guarantee applies under the assumptions stated above.

Formally, define
\[
    h_i(x)=\mathbb E_Z[w_i(x+\sqrt q\,Z)],
\]
the probability that one complete cycle starting from $x$ recovers $x_i$.
We now bound this probability and the resulting finite-cycle residence.

\subsection{Proof of Theorem~\ref{thm:recovery}}
\label{app:proof-recovery}

\paragraph{Forward process.}
The variance-preserving forward SDE
\[
    dY_t=-Y_t\,dt+\sqrt2\,dB_t,
    \qquad Y_0=x,
\]
where $B_t$ is standard Brownian motion, has transition law
\[
    Y_t\overset{d}{=}e^{-t}x+\sqrt{1-e^{-2t}}\,Z,
    \qquad Z\sim\mathcal N(0,I).
\]
For $t>0$, setting $q=e^{2t}-1$ and $U=e^tY_t$ gives
\[
    U\overset{d}{=}x+\sqrt q\,Z.
\]
Thus the additive-noise formulation is a rescaling of this forward
process. We fix $q>0$ across cycles.

\paragraph{Assumption~\ref{assump:exact-score} .}
The clean law is $N^{-1}\sum_{j=1}^N\delta_{x_j}$, with distinct
$x_1,\ldots,x_N$ and $N\ge2$. Denoising is its exact stochastic reverse
diffusion to time zero, using the exact empirical score. Each cycle
uses fresh independent forward noise and reverse-sampling randomness.

For $q>0$, define
\begin{equation}
    w_i(u)=
    \frac{e^{-\|u-x_i\|^2/(2q)}}
         {\sum_j e^{-\|u-x_j\|^2/(2q)}},
    \qquad
    E_q(u)=-q\log\sum_j e^{-\|u-x_j\|^2/(2q)}.
    \label{eq:app-weights-energy}
\end{equation}
Here $E_q$ is the scaled negative log-density, up to an additive
constant, and
\[
    -\nabla E_q(u)=\sum_jw_j(u)x_j-u.
\]
Bayes' rule gives
$\mathbb P(X'=x_i\mid U=u)=w_i(u)$.
This endpoint law describes posterior sampling, rather than the
posterior mean or an arbitrary finite-step reverse sampler.

\paragraph{Theorem~\ref{thm:recovery} (restated).}
Fix $x_i$, write $d_{ij}=\|x_j-x_i\|$, and choose
\[
    0<4R<\min_{j\ne i}d_{ij}.
\]
Let $B_R=\{x:\|x-x_i\|\le R\}$. For $0\le r\le R$, define
\[
    A_i(r)=
    \min\left\{
        1,\frac12\sum_{j\ne i}
        \exp\left[-\frac{d_{ij}(d_{ij}-4r)}{8q}\right]
    \right\},
    \qquad
    h_i(x)=\mathbb E_Z[w_i(x+\sqrt q\,Z)].
\]
Then $h_i(x)\ge1-A_i(r)$ whenever $\|x-x_i\|\le r$.
Moreover, for $X_0=x\in B_R$ and every integer $M\ge1$,
\[
    \mathbb P_x(X_1=\cdots=X_M=x_i)
    =h_i(x)\,h_i(x_i)^{M-1}
    \ge[1-A_i(R)]\,[1-A_i(0)]^{M-1}.
\]

\begin{proof}
We first bound the probability of recovering a competing training point,
then use the Markov property to obtain the residence bound.

\medskip\noindent
\textit{(i) A two-point bound.}
Fix $j\ne i$ and write
\[
    L_k(u)=e^{-\|u-x_k\|^2/(2q)}.
\]
The denominator of $w_j(u)$ contains the two distinct positive terms
$L_i(u)$ and $L_j(u)$, so
\[
    w_j(u)\le\frac{L_j(u)}{L_i(u)+L_j(u)}.
\]
Writing $\tau=L_j(u)/L_i(u)>0$ and using
$1+\tau\ge2\sqrt\tau$ gives
\begin{equation}
    w_j(u)\le\frac{\tau}{1+\tau}
    \le\frac12\sqrt\tau
    =\frac12\sqrt{\frac{L_j(u)}{L_i(u)}}.
    \label{eq:proof-sqrt-bound}
\end{equation}

\medskip\noindent
\textit{(ii) Recovery after Gaussian noising.}
Expanding the squared norms yields
\[
    \|u-x_j\|^2-\|u-x_i\|^2
    =d_{ij}^2-2\langle x_j-x_i,u-x_i\rangle.
\]
Consequently, for $U=x+\sqrt q\,Z$,
\[
    \sqrt{\frac{L_j(U)}{L_i(U)}}
    =
    \exp\left[
        -\frac{d_{ij}^2}{4q}
        +\frac{\langle x_j-x_i,x-x_i\rangle}{2q}
        +\frac{\langle x_j-x_i,Z\rangle}{2\sqrt q}
    \right].
\]
The Gaussian identity
$\mathbb E[e^{\langle v,Z\rangle}]=e^{\|v\|^2/2}$
therefore gives
\begin{equation}
    \mathbb E\sqrt{\frac{L_j(U)}{L_i(U)}}
    =
    \exp\left[
        -\frac{d_{ij}^2}{8q}
        +\frac{\langle x_j-x_i,x-x_i\rangle}{2q}
    \right].
    \label{eq:proof-gaussian-moment}
\end{equation}
If $\|x-x_i\|\le r$, Cauchy--Schwarz implies
$\langle x_j-x_i,x-x_i\rangle\le d_{ij}r$, hence
\[
    \mathbb E\sqrt{\frac{L_j(U)}{L_i(U)}}
    \le
    \exp\left[-\frac{d_{ij}(d_{ij}-4r)}{8q}\right].
\]
Combining this with~\eqref{eq:proof-sqrt-bound} and summing over
$j\ne i$ gives
\[
    1-h_i(x)
    =\mathbb E\left[\sum_{j\ne i}w_j(U)\right]
    \le
    \frac12\sum_{j\ne i}
    \exp\left[-\frac{d_{ij}(d_{ij}-4r)}{8q}\right].
\]
Since also $1-h_i(x)\le1$, we conclude that
\[
    1-h_i(x)\le
    \min\left\{
        1,\frac12\sum_{j\ne i}
        \exp\left[-\frac{d_{ij}(d_{ij}-4r)}{8q}\right]
    \right\}
    =A_i(r).
\]
This proves the one-cycle recovery bound.

\medskip\noindent
\textit{(iii) Cyclic residence.}
Fresh independent randomness in each cycle gives the transition law
\[
    \mathbb P(X_{k+1}=x_i\mid X_k=y)=h_i(y).
\]
Although $X_0=x$ may lie outside the training set, every $X_k$ with
$k\ge1$ belongs to $\mathcal D$. By the Markov property,
\[
    \mathbb P_x(X_1=\cdots=X_M=x_i)
    =h_i(x)\prod_{k=2}^M h_i(x_i)
    =h_i(x)\,h_i(x_i)^{M-1},
\]
where the product is empty when $M=1$.
Applying part (ii) at radius $R$ to the first factor and at radius
$0$ to the subsequent factors gives
\[
    \mathbb P_x(X_1=\cdots=X_M=x_i)
    \ge[1-A_i(R)]\,[1-A_i(0)]^{M-1}.
\]
For $M=1$, the second factor is interpreted as $1$.
\end{proof}

\subsection{Experimental Verification}
\label{sec:proof-experimental}
We check whether the quantitative prediction of Theorem~\ref{thm:recovery} matches exact empirical cycling, and
whether a learned model exhibits corresponding recovery and
residence near training images. Defining,
$\Delta_i=\min_{j\ne i}\|x_j-x_i\|$.
Nearby initializations are displaced by $0.2\Delta_i$ toward the
target's nearest training neighbor, satisfying the theorem's
condition $4R_i<\Delta_i$ with $R_i=0.2\Delta_i$.
The experimental cycling level $\gamma$ specifies a fraction
of the forward schedule; its corresponding additive variance
is $q=(1-\bar\alpha_t)/\bar\alpha_t$.

\paragraph{Testing the quantitative prediction.}
We use an eight-dimensional bimodal GMM dataset containing
2048 training points and select 20 targets spanning
nearest-neighbor separations. Each exact cycle adds Gaussian
noise and samples a training point from the posterior weights
$w_j(U)$ defined above. For each target and noise level, we
estimate one-cycle return probabilities using 20,000 trials.
Separately, for each target and noise level, we simulate
2048 independent trajectories from nearby initializations
for up to 500 cycles, counting residence
only when every clean endpoint equals the original target.

This allows an independent comparison between measured
residence, the theorem's product prediction
$h_i(x)h_i(x_i)^{M-1}$ estimated from one-cycle trials, and
its lower bound $[1-A_i(R_i)][1-A_i(0)]^{M-1}$.
At $\gamma=0.05$, measured 500-cycle residence averages
$66.48\%$, closely matching the product estimate of $66.01\%$,
while the mean lower bound is $45.25\%$.
At $\gamma=0.10$, these values are $9.72\%$, $9.57\%$,
and $3.06\%$, respectively.
All averages are taken over target-wise probabilities.
Thus, the experiment supports the quantitative prediction
and shows that the bound gives a nontrivial guarantee
of residence in the tested low-noise regime.

\paragraph{Recovery and residence in a learned model.}
We next compare two checkpoints of the same U-Net(Adam Optimizer) trained
on 512 CelebA images: an early checkpoint at $\tau=2400$
and a late checkpoint at $\tau=130000$, with logged
memorization fractions of $0.69\%$ and $100.00\%$,
respectively. We report the comparison at $\gamma=0.05$,
using the same 20 separation-ranked targets at both
checkpoints and 16 trajectories per target from nearby
initializations. Outputs are passed directly into subsequent
cycles without projection onto the training set.

For the late checkpoint, mean relative reconstruction
distance $\|X_k-x_i\|/\Delta_i$ falls from $0.20$ initially
to $0.069$ after one cycle and remains approximately
$0.070$ after 100 cycles. For the early checkpoint, it
instead increases to $0.314$ after one cycle and $1.468$
after 100 cycles. Thus, on average over the tested initializations, the late checkpoint rapidly recovers proximity to the target and maintains a small reconstruction error, whereas the early
checkpoint's error grows. To distinguish accurate reconstruction from merely retaining
the same nearest training image, we measure uninterrupted
residence within a specified reconstruction tolerance.
At the prespecified sensitivity tolerance of $0.10\Delta_i$,
$86.88\%$ of late-checkpoint trajectories remain within
the target neighborhood at every endpoint through 100
cycles, compared with no observed survivors for the early
checkpoint. At the stricter primary tolerance of
$0.05\Delta_i$, late-checkpoint residence is $2.81\%$.
The contrast between these tolerances also highlights the
remaining reconstruction error where the late checkpoint's
mean distance from the target is approximately
$0.07\Delta_i$ after both one and 100 cycles.

Together, the experiments support the recovery and residence
mechanism of Theorem~\ref{thm:recovery}. Exact empirical
sampling corroborates its quantitative prediction. Under
the same perturbations and cycling noise, the fully
memorized checkpoint recovers proximity to training targets
and sustains it, whereas the checkpoint with low memorization
moves away. This contrast supports the interpretation of
memorized training images as locally stable basins under
cyclic denoising. The dependence on reconstruction tolerance
reflects approximate recovery in the learned model, while
the theorem provides a quantitative guarantee under exact
empirical reversal.

\section{Additional Results}
\label{app:additional-results}

\subsection{Results on the GMM}
\label{app:gmm-results}
 
The GMM has a known population score and low dimension ($d=8$), so every
quantity in the paper can be computed on it cheaply and, for the landscape
plots, exactly. We use it to check that the signatures observed on CelebA
and CIFAR-10 also appear in this controlled setting. All models are residual MLPs trained with full-batch SGD
(App.~\ref{app:finer-things-club}).
 
\paragraph{Geometric signatures across training.}
Figure~\ref{fig:gmm-training-dynamics} repeats the analysis of
Fig.~\ref{fig:training-dynamics} on the GMM. The held-out loss and
$D_{\mathrm{KL}}(P_\theta\|P_0)$ locate the overfitting transition
(Fig.~\ref{fig:gmm-training-dynamics}A). The normalized score divergence
$\widehat{\mathcal A}_1$ is close to zero and identical for training,
held-out, and generated samples in the generalized regime, and the
training and generated curves then fall well below the held-out curve
(Fig.~\ref{fig:gmm-training-dynamics}B). On the GMM, the lag between this separation and
the rise in $f_{\mathrm{mem}}$ is smaller than on CelebA.
As on CelebA, both occur later for larger $n$, and replotting against
$\tau/n$ aligns the onsets (inset). Basin volume separates in the same
direction: it grows by several orders of magnitude for training and
generated samples and does not grow for held-out samples
(Fig.~\ref{fig:gmm-training-dynamics}C). The generated curves follow the
training curves in both signatures.
 
\begin{figure*}[ht]
    \centering
    \includegraphics[width=\textwidth]{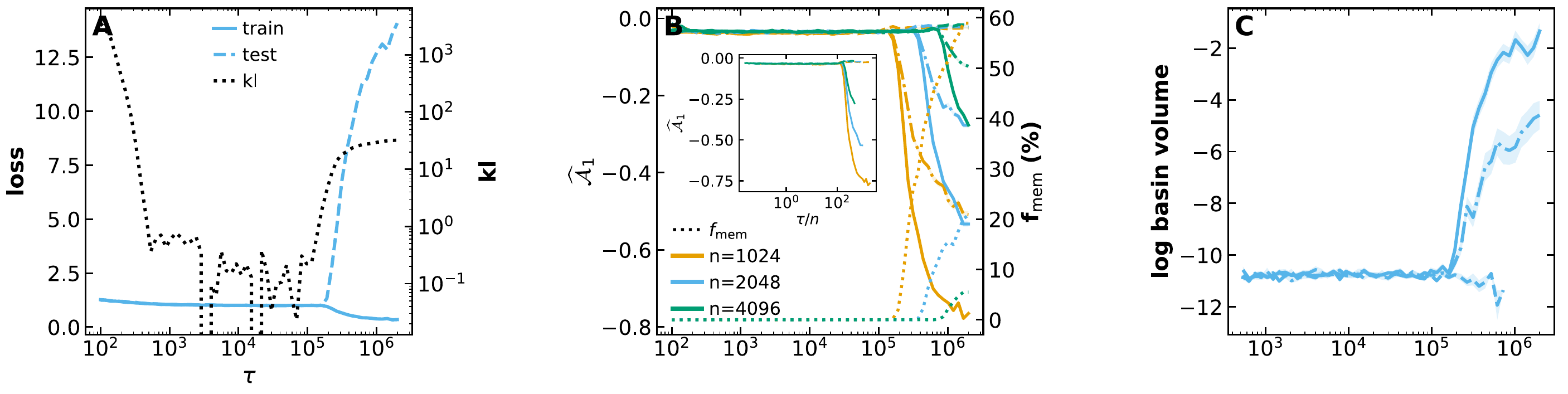}
    \vspace{-0.2in}
    \caption{\textbf{Geometric signatures across training on the GMM.}
    \textbf{(A)} Train loss (solid) and test loss (dashed) at $\gamma=0.01$
    on the left axis, and $D_{\mathrm{KL}}(P_\theta\|P_0)$ (dotted, log
    scale) on the right axis, versus training step $\tau$, for $n=2048$.
    \textbf{(B)} Mean normalized score divergence
    $\widehat{\mathcal A}_1$ on the left axis and $f_{\mathrm{mem}}$
    (dotted) on the right axis, versus $\tau$. Line style denotes the sample
    group: train (solid), test (dashed), generated (dash-dot). Colour
    denotes the training-set size: $n=1024$ (orange), $n=2048$ (blue),
    $n=4096$ (green). \emph{Inset:} the same divergence curves against
    $\tau/n$. \textbf{(C)} Mean log basin volume versus $\tau$ for $n=2048$,
    with the line styles of (B); bands show one standard deviation across
    samples.}
    \label{fig:gmm-training-dynamics}
\end{figure*}
 
\paragraph{Cyclic denoising from generated and training samples.}
Figure~\ref{fig:gmm-cyclic-population} applies CD to the $n=512$ model at
four checkpoints, from the generalized regime ($\tau=42{,}107$,
one-shot $f_{\mathrm{mem}}=0.01\%$) to the memorized regime
($\tau=864{,}044$, $85.47\%$). We start $512$ trajectories from fresh
generations and $512$ from the training points, and cycle them at eleven
amplitudes between $\gamma=0.001$ and $\gamma=0.35$.
 
Cycling exposes memorization that one-shot sampling misses. At
$\tau=97{,}466$, one-shot $f_{\mathrm{mem}}$ is $4.3\%$, and $100$ cycles
at $\gamma\in[0.02,0.06]$ raise it to about $20\%$
(Fig.~\ref{fig:gmm-cyclic-population}A). At $\tau=161{,}268$, cycling at
$\gamma=0.02$ raises it from $45\%$ to $94\%$. At large amplitudes
($\gamma\ge0.16$) $f_{\mathrm{mem}}$ stays near its one-shot value, so, as on
CelebA, basins become accessible only when the perturbation is small
enough. In the generalized regime no basin is reached at any amplitude,
and even trajectories started exactly on training points leave them: their
$f_{\mathrm{mem}}$ falls from $100\%$ to $0\%$ within ten cycles at every
$\gamma\ge0.04$ (Fig.~\ref{fig:gmm-cyclic-population}C). Training points
are therefore not stable states of this model. At later checkpoints, the
trajectories started from training points first lose some memorized states
and then return to the same level as those started from generated samples,
so both kinds of starting points reach the same set of basins.
 
The cosine similarity to the starting sample shows how far trajectories
travel (Fig.~\ref{fig:gmm-cyclic-population}B, D). At every checkpoint it
decays faster at larger $\gamma$, as expected from stronger perturbations.
With training, trajectories at small $\gamma$ stay closer to their start:
after $100$ cycles at $\gamma=0.02$, the similarity is $0.75$ at
$\tau=42{,}107$ and $1.00$ at $\tau=864{,}044$. This matches the
residence predicted by Theorem~\ref{thm:recovery} for small cycling
variance. At $\gamma=0.35$ the similarity falls to about zero at every
checkpoint, yet in the memorized regime $f_{\mathrm{mem}}$ remains near
$85\%$. The trajectories therefore keep landing on training points but no
longer on the one they started from. This is the large-$q$ limit of
Theorem~\ref{thm:recovery}, in which each cycle ends at an essentially
random training point.
 
\begin{figure*}[ht]
    \centering
    \includegraphics[width=\textwidth]{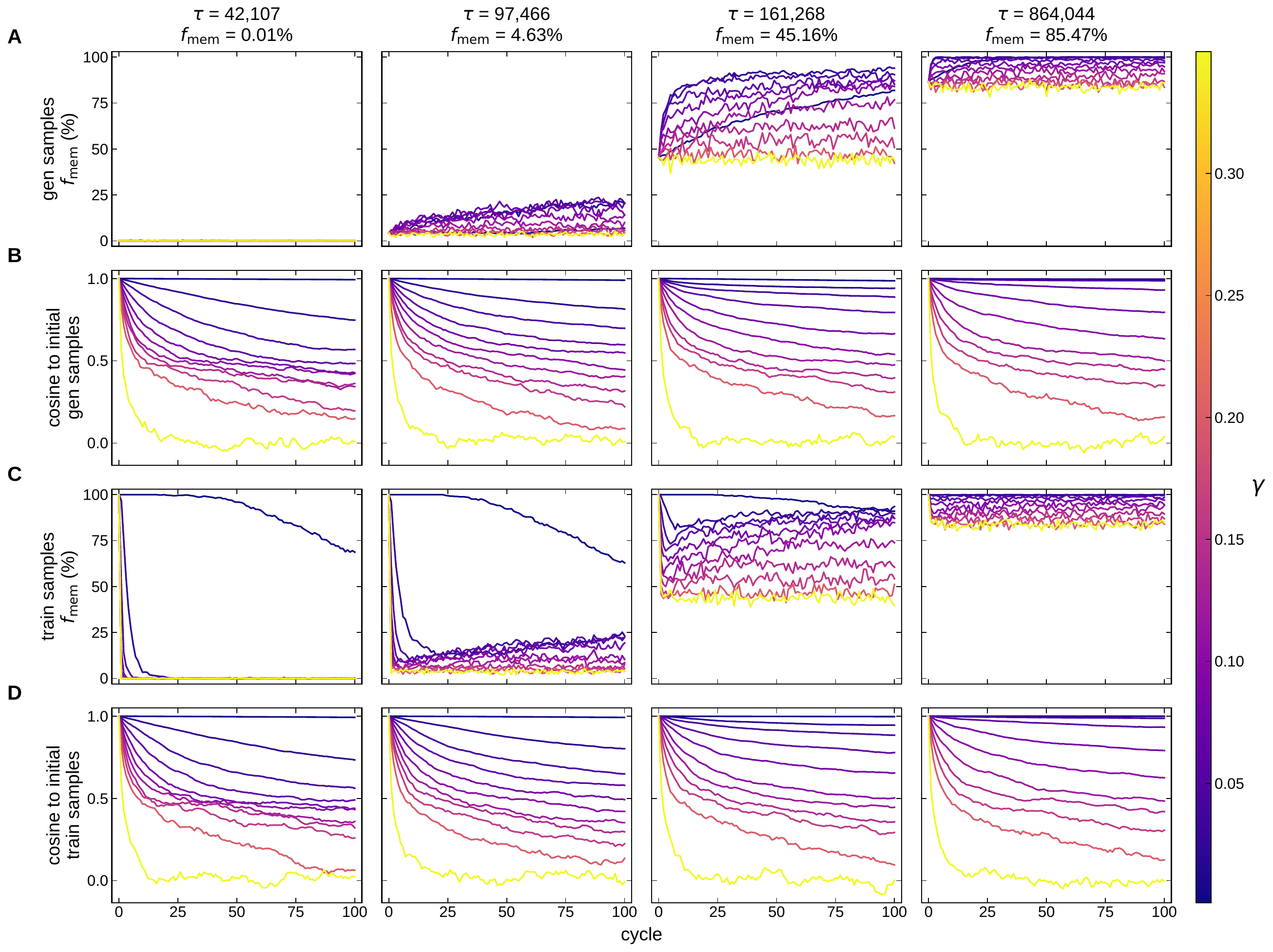}
    \vspace{-0.2in}
    \caption{\textbf{Cyclic denoising on the GMM across training.}
    GMM, $n=512$. Columns are checkpoints of increasing $\tau$, with their
    one-shot $f_{\mathrm{mem}}$. Each panel shows $512$ trajectories per
    amplitude, coloured by $\gamma$ ($0.001\le\gamma\le0.35$).
    \textbf{(A)} $f_{\mathrm{mem}}$ versus cycle for trajectories started
    from generated samples. \textbf{(B)} Mean cosine similarity between each
    state and its starting sample, computed on the raw $8$-dimensional vectors, for the same trajectories.
    \textbf{(C, D)} The same quantities for trajectories started from the
    training points. Cycling raises $f_{\mathrm{mem}}$ above its one-shot
    value at small $\gamma$ once basins have formed, while at
    $\tau=42{,}107$ even training points are not retained.}
    \label{fig:gmm-cyclic-population}
\end{figure*}
 
\paragraph{Trajectories on the score-divergence landscape.}
Figure~\ref{fig:gmm-cyclic-trajectories} follows a single trajectory on
the landscape. Because the GMM score network is cheap to differentiate, we
compute $\mathcal A_1$ exactly as the trace of the Jacobian of the learned
score at $t=1$, on a two-dimensional slice built as in
App.~\ref{app:landscape-viz}: a $200\times200$ grid
$x(\alpha,\beta)=x^\star+\alpha\,\mathbf d_1+\beta\,\mathbf d_2$.
Here the origin $x^\star$ is the midpoint of the two GMM components,
$\mathbf d_1$ is the direction $\boldsymbol{\mu}/\|\boldsymbol{\mu}\|$ joining the two component means,
and $\mathbf d_2$ is the leading principal direction of the training data
orthogonal to $\mathbf d_1$. This fixes the plane so that it contains
both components at every checkpoint and is the same across rows. Each cycle
state is projected onto the plane. The
landscape changes in both scale and shape during training. In the
generalized regime, $\mathcal A_1$ is smooth and small in magnitude, with
a ridge between the two GMM components. After memorization sets in,
its magnitude grows by two orders of magnitude and it breaks into many
small regions of strongly negative divergence, the localized basins
around individual training points. The trajectories follow the behaviour
of Fig.~\ref{fig:gmm-cyclic-population}. At small $\gamma$ they stay near
their start, and at $\tau=864{,}044$ with $\gamma=0.10$ the trajectory
moves into one such region and stays there. At large $\gamma$ they jump
between distant regions of the plane.
 
\begin{figure*}[ht]
    \centering
    \includegraphics[width=\textwidth]{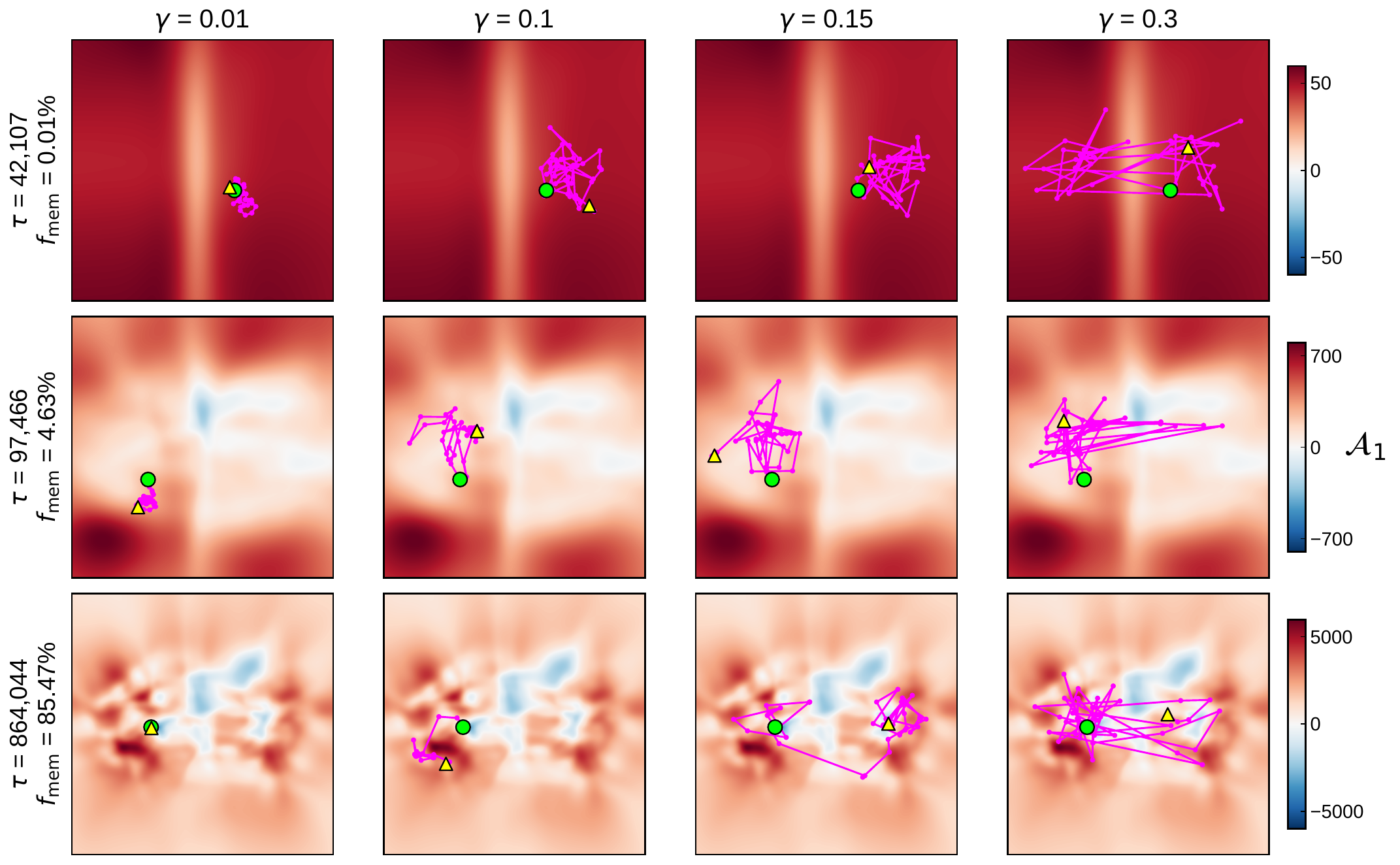}
    \vspace{-0.2in}
    \caption{\textbf{Cyclic-denoising trajectories on the GMM score-divergence
    landscape.} GMM, $n=512$. Rows are checkpoints of increasing $\tau$,
    with their one-shot $f_{\mathrm{mem}}$; columns are cycling amplitudes
    $\gamma$. The background is the exact score divergence
    $\mathcal A_1$ of the learned score on a two-dimensional plane spanned
    by the direction joining the two component means and the leading orthogonal principal
    direction of the training data; red marks negative (contracting)
    values. Each row has its own colour scale, since the range grows from
    about $\pm60$ to $\pm6{,}000$ during training. The magenta path joins
    the projected states of one trajectory every $5$ cycles over $200$
    cycles, from the start (green dot) to the end (yellow triangle).
    Because states are projected onto the plane, the path shows direction
    of motion rather than exact position on the background.}
    \label{fig:gmm-cyclic-trajectories}
\end{figure*}
\FloatBarrier

\subsection{Cyclic Denoising}
Here, we provide futher results of CD dynamics.

\paragraph{Degenerate attractors}
\label{sec:degenerate-attractors}
Figures~\ref{fig:cd-saturated}--\ref{fig:cd-celeba} illustrate
saturated colour blocks, stripe-like textures, nearly monochromatic
fields, and loss of facial detail. We refer to these as
\emph{degenerate attractor candidates}. We identify non-memorized
residence by thresholding high consecutive similarity,
together with the absence of memorization at every logged cycle. Consistent with the population analysis in
Fig.~\ref{fig:cyclic-population}, residence alone does not establish
memorization. On CIFAR-10 at $\gamma=0.05$, the fraction satisfying
this screen decreases from $50.87\%$ at
$\tau_{\mathrm{gen}}=49{,}500$ to $16.00\%$ at the
generalization-window checkpoint $\tau=83{,}000$, and to $0.98\%$
at $\tau=200{,}000$. This corresponds to a decrease of $49.89$
percentage points, or approximately a $98.1\%$ relative reduction. At $\tau_{\mathrm{gen}}=49{,}500$, the fraction of CIFAR-10
trajectories passing the same residence and non-memorization
screen falls from $50.87\%$ at $\gamma=0.05$ to $2.54\%$ at
$\gamma=0.10$, with none passing at $\gamma=0.20$ or $0.30$.
Alongside increasing $f_{\mathrm{mem}}$ under cycling, this is
consistent with residence shifting towards training-sample basins
as training progresses. Qualitatively similar degeneration
is visible on CelebA, where selected trajectories darken or become
washed out, losing facial detail under repeated CD
(Fig.~\ref{fig:cd-celeba}).

Their appearance ranges from stripes and nearly monochromatic
fields to configurations retaining recognizable objects.
The CelebA examples similarly illustrate that residence can accompany
substantial loss of facial detail without satisfying the memorization
criterion; they do not establish a monotonic decline across training.

Throughout, $f_{\mathrm{mem}}$ measures memorization rather than
residence. Where both are reported, we distinguish
$f_{\mathrm{mem}}$ for one-shot generation from
$f_{\mathrm{mem}}$ after $500$ CD cycles. These are population
measurements, separate from the displayed non-memorized trajectories.
Together, the results distinguish residence, characterized by
$S_{\mathrm{cyc}}^{(k)}$ and $\mathcal{A}_1$, from identification
with a training sample through $f_{\mathrm{mem}}$.
The declining share of degenerate attractors on CIFAR-10 accompanies
increasing memorization under cycling, while the mechanism by which
these basins arise and lose accessibility remains open.
For clarity, we distinguish the logged ordinary-generation checkpoint estimate, $f_{\mathrm{mem}}^{\mathrm{gen}}$, from the logged full-generated-cohort estimate after 500 cycles, $f_{\mathrm{mem}}^{\mathrm{CD}}(500)$. The percentages below are the recorded bootstrap mean estimates; they are not prevalence estimates for the visually selected examples.

\paragraph{Saturated geometric configurations.}
At $\tau=49{,}500$, CIFAR-10 trajectories evolve from recognizable images into saturated colour blocks and simplified geometric configurations (Figure~\ref{fig:cd-saturated}). Here, $f_{\mathrm{mem}}^{\mathrm{gen}}=0\%$ and $f_{\mathrm{mem}}^{\mathrm{CD}}(500)=0\%$. Thus, substantial visual degeneration occurs without triggering the pixel-distance memorization criterion. These six examples were selected by their highest late similarity among qualifying trajectories.

\begin{figure}[htbp]
\centering
\includegraphics[width=\linewidth]{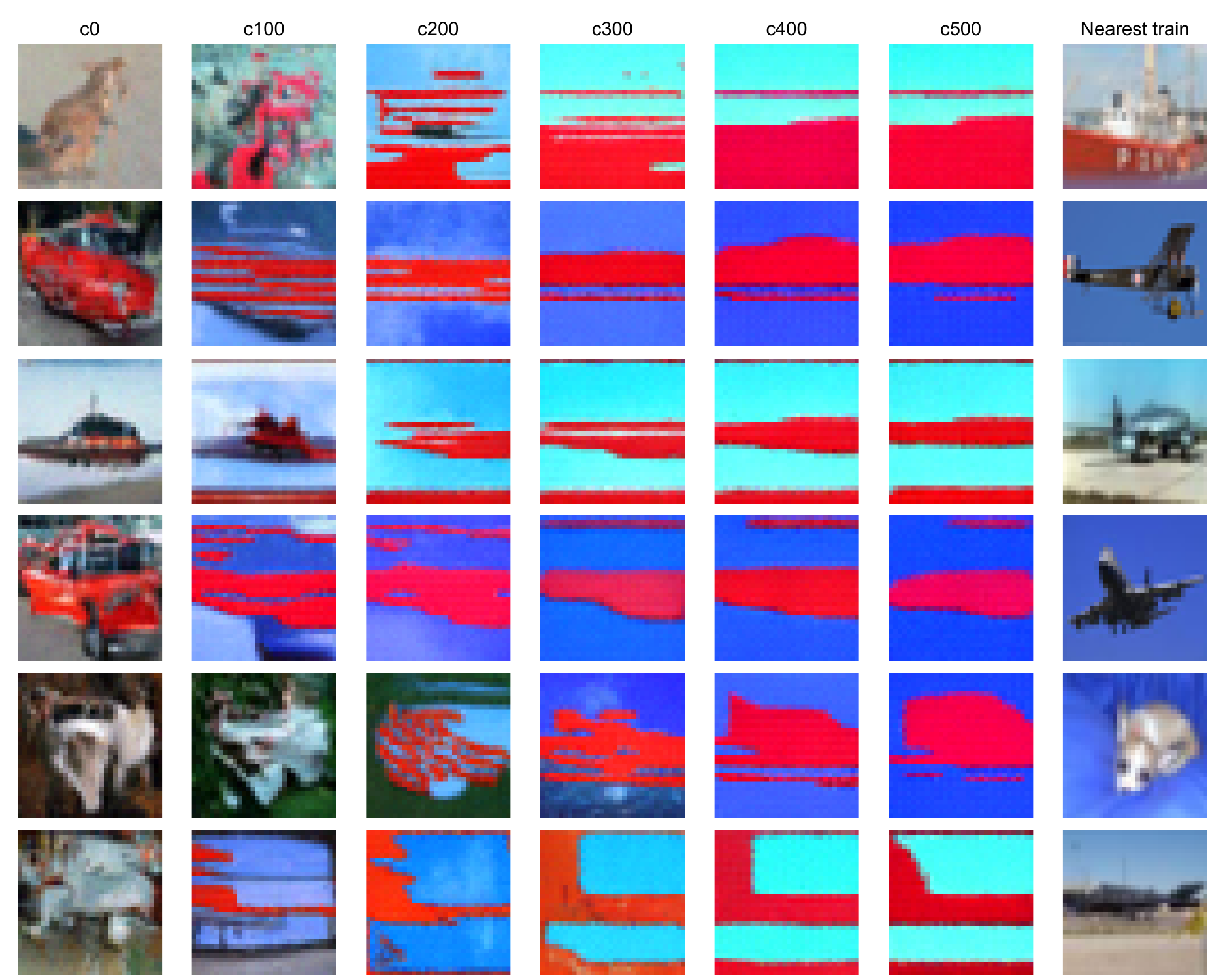}
\caption{\textbf{Saturated geometric degeneration on CIFAR-10.} $\tau=49{,}500$, $\gamma=0.05$; $f_{\mathrm{mem}}^{\mathrm{gen}}=0\%$ and full-cohort $f_{\mathrm{mem}}^{\mathrm{CD}}(500)=0\%$. Repeated CD replaces recognizable content with saturated bands and colour blocks. The final column shows the nearest training image to c500.}
\label{fig:cd-saturated}
\end{figure}

\paragraph{Fading structure and nearly monochromatic endpoints.}
At $\tau=83{,}000$, the selected CIFAR-10 trajectories lose object detail and end in low-contrast blue fields (Figure~\ref{fig:cd-fading}). The checkpoint has $f_{\mathrm{mem}}^{\mathrm{gen}}=0\pm0\%$, while the full generated cohort reaches $f_{\mathrm{mem}}^{\mathrm{CD}}(500)=3.2\%$. The displayed trajectories remain outside the memorization criterion. With high late cycles mean similarity should not be interpreted as proof that the final appearance has stabilized.

\begin{figure}[htbp]
\centering
\includegraphics[width=\linewidth]{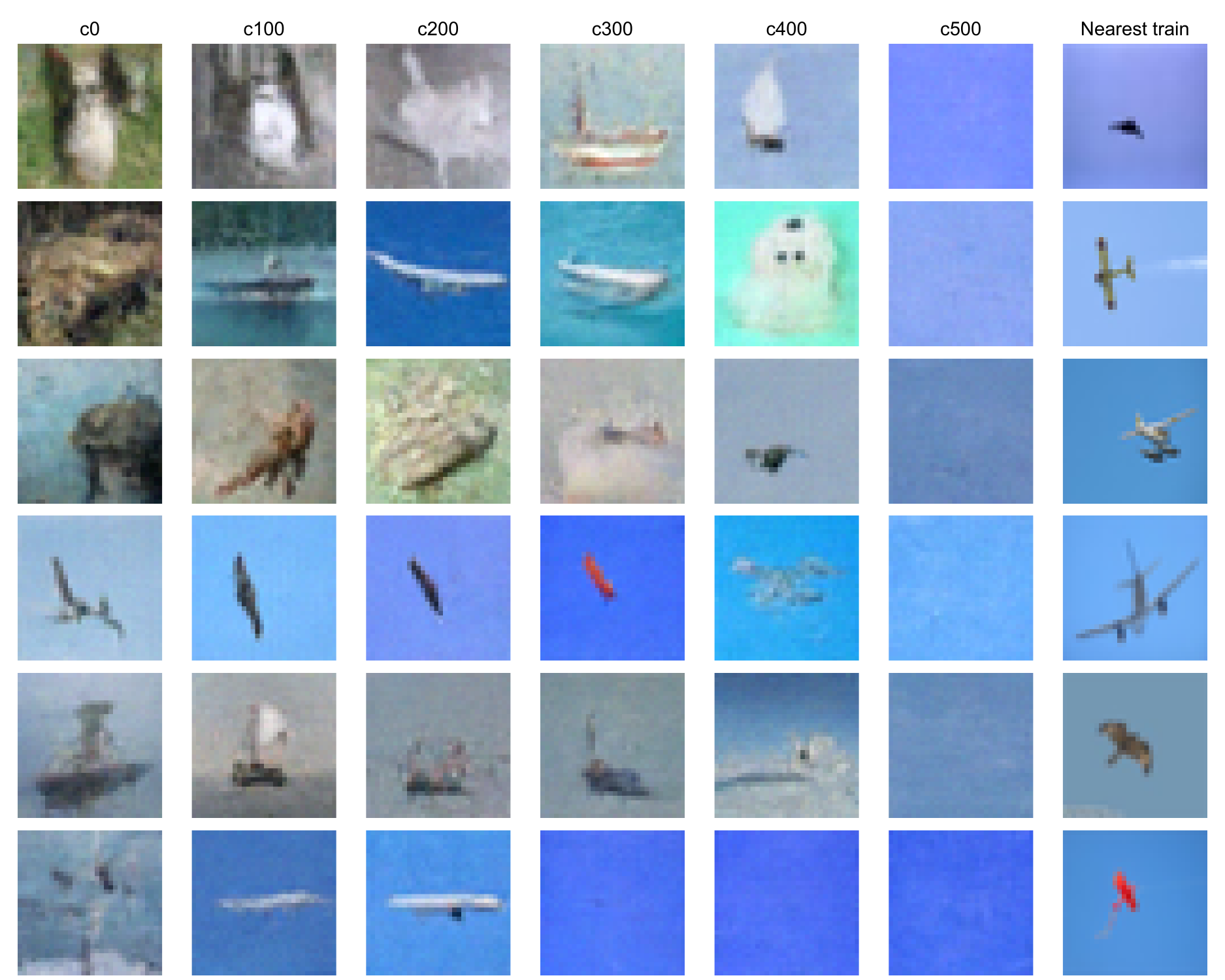}
\caption{\textbf{Fading structure on CIFAR-10.} $\tau=83{,}000$, $\gamma=0.05$; $f_{\mathrm{mem}}^{\mathrm{gen}}=0\%$ and full-cohort $f_{\mathrm{mem}}^{\mathrm{CD}}(500)=3.1921\%$. Six appearance-selected trajectories lose semantic structure and approach nearly monochromatic blue configurations.}
\label{fig:cd-fading}
\end{figure}

\paragraph{Stripe-like textures and colour fields.}
Additional trajectories at $\tau=49{,}500$ develop repeated stripes, colour bands, or nearly uniform fields (Figure~\ref{fig:cd-stripes}). Both $f_{\mathrm{mem}}^{\mathrm{gen}}$ and $f_{\mathrm{mem}}^{\mathrm{CD}}(500)$ are $0\%$. The intermediate states reveal the emergence of these patterns and subsequent changes in colour and contrast.

\begin{figure}[htbp]
\centering
\includegraphics[width=\linewidth]{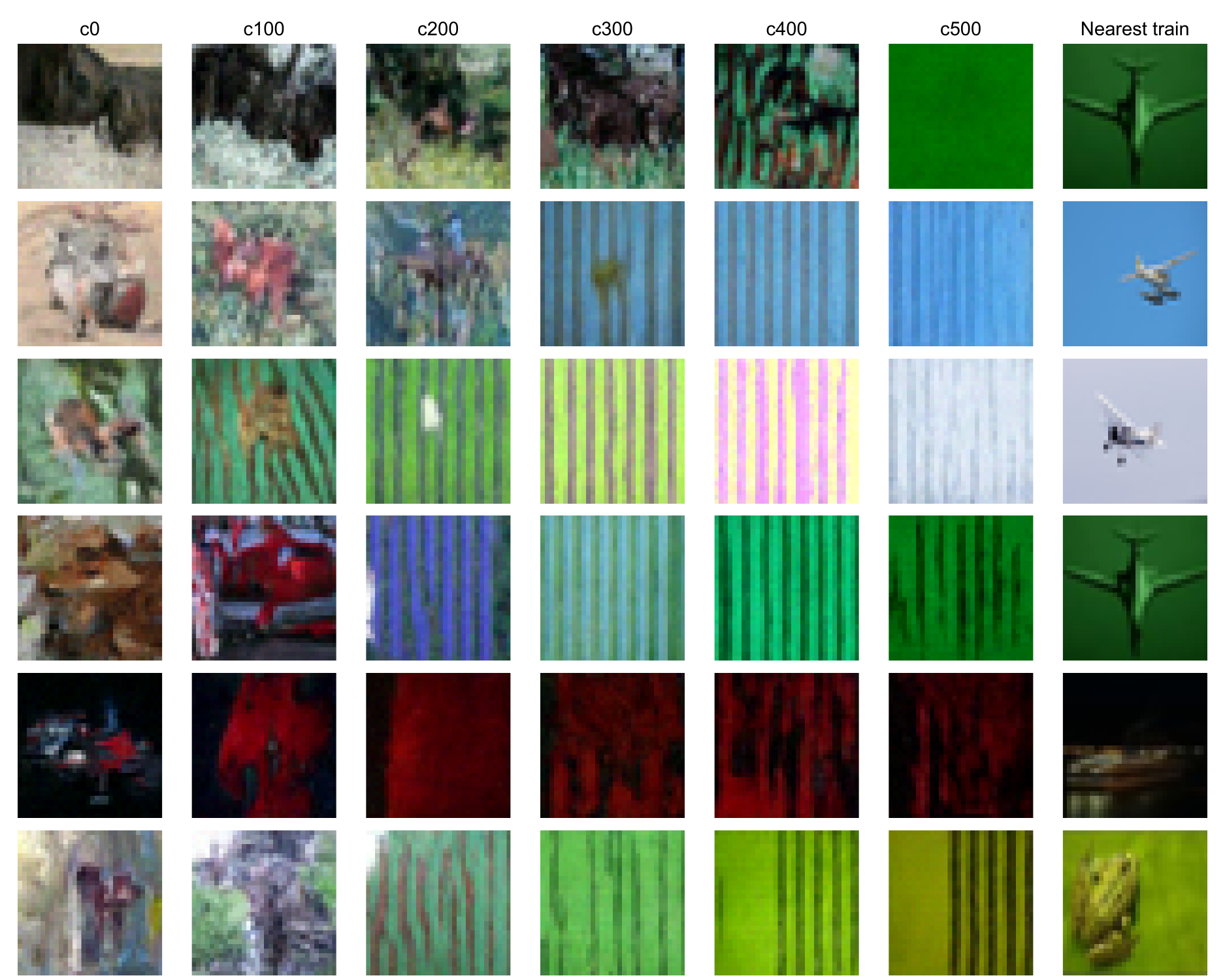}
\caption{\textbf{Stripe-like and monochromatic degeneration on CIFAR-10.} $\tau=49{,}500$, $\gamma=0.05$; $f_{\mathrm{mem}}^{\mathrm{gen}}=0\%$ and full-cohort $f_{\mathrm{mem}}^{\mathrm{CD}}(500)=0\%$. Trajectories develop repeated spatial bands or nearly monochromatic fields.}
\label{fig:cd-stripes}
\end{figure}

\paragraph{CelebA across training checkpoints.}
Figure~\ref{fig:cd-celeba} combines two examples from each of three checkpoints at $\gamma=0.05$. At $\tau=5{,}200$ (top pair), faces darken and lose detail; both memorization estimates are $0\%$. At $\tau=10{,}000$ (middle pair), the selected faces become bright and washed out, with $f_{\mathrm{mem}}^{\mathrm{gen}}=0\%$ and full-cohort $f_{\mathrm{mem}}^{\mathrm{CD}}(500)=9.8\%$. At $\tau=20{,}000$ (bottom pair), the selected outputs retain coherent facial structure, while the corresponding estimates are $0.1308\%$ and $21.1202\%$. All six displayed trajectories remain non-memorized under the stated rule at the logged cycles.

\begin{figure}[htbp]
\centering
\includegraphics[width=\linewidth]{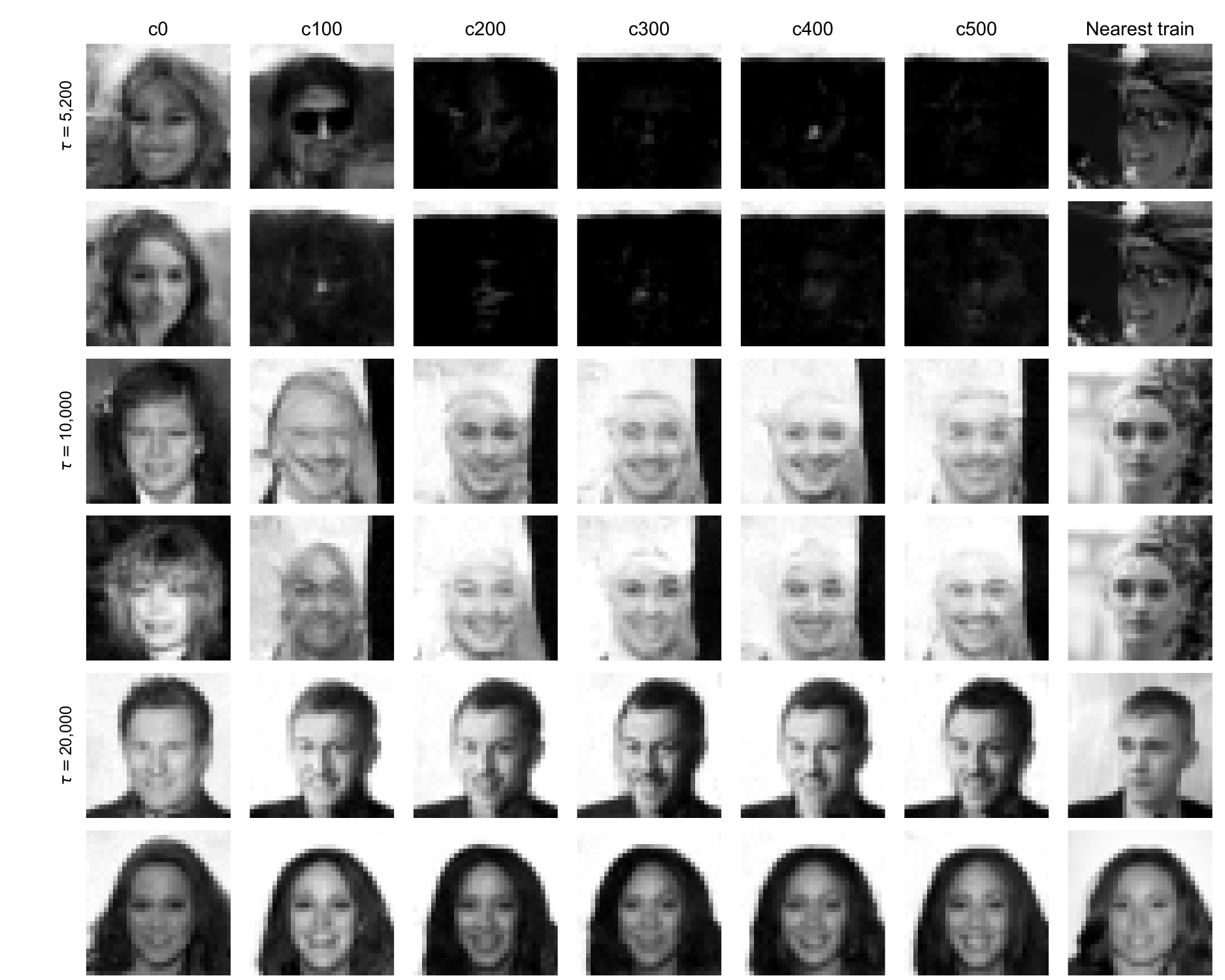}
\caption{\textbf{CelebA CD trajectories across training checkpoints, $\gamma=0.05$.} Top, middle and bottom pairs correspond to $\tau=5{,}200$, $10{,}000$ and $20{,}000$, respectively. Their logged $f_{\mathrm{mem}}^{\mathrm{gen}}$ values are $0\%$, $0\%$ and $0.1308\%$; full-cohort $f_{\mathrm{mem}}^{\mathrm{CD}}(500)$ values are $0\%$, $9.8\%$ and $21.12\%$. At $\tau=5200$ the images move towards mere monochromatic darkness and remain stable there. For $\tau=10000$ and $\tau=20000$ more facial structure emerges for these non memorized basins.}
\label{fig:cd-celeba}
\end{figure}

\begin{figure}[ht]
    \centering
    \includegraphics[width=0.8\linewidth]{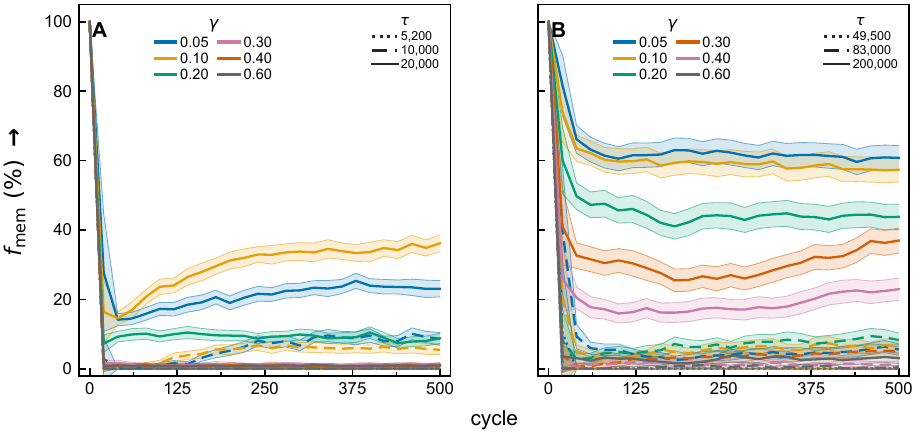}
    \caption{
    \textbf{Cyclic denoising from training samples.}
    Memorization fraction $f_{\mathrm{mem}}$ across 500 cycles for
    trajectories initialized from training images in
    \textbf{(A)} CelebA ($n=4096$, Adam, linear noise schedule) and
    \textbf{(B)} CIFAR-10 ($n=8192$, Adam, cosine noise schedule).
    Although $f_{\mathrm{mem}}$ starts at $100\%$, it drops sharply
    under cycling. At $\tau_{\mathrm{gen}}$, it remains near zero,
    whereas later checkpoints retain or recover a larger fraction,
    depending on $\gamma$. Classification is against the full training
    set and does not require returning to the initial training image.
   Shaded bands combine the reported
    bootstrap standard errors with variation within each cycle bin. The memory of the initial train samples can be lost under cycling, basin residence only remains in a subset of train samples that become attractors(subject to cycling amplitude). Cycling moves trajectories to other stronger attractor basins.
    }
    \label{fig:cyclic-train-fmem}
\end{figure}

\begin{figure*}[ht]
    \centering
    \includegraphics[width=\linewidth]{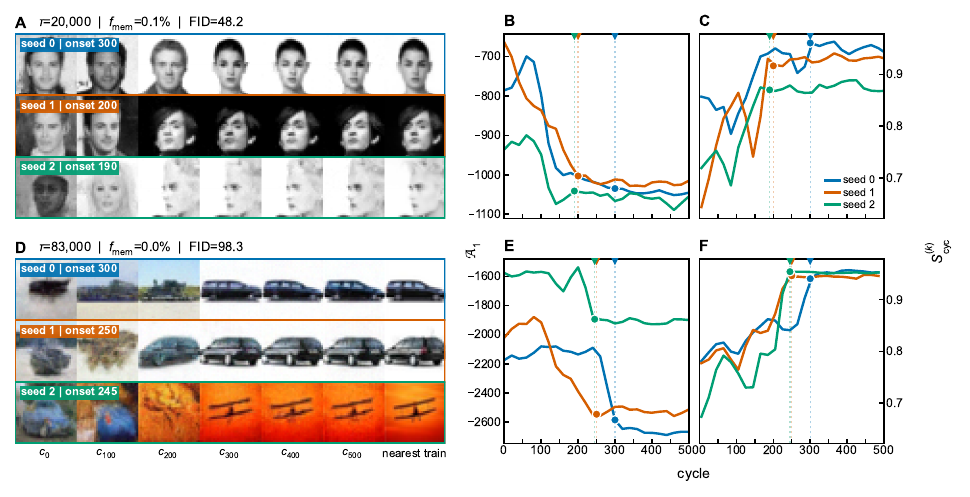}
 \caption{
\textbf{Trajectories discover and reside in basins corresponding to training samples.}
(\textbf{A}) CelebA trajectories ($n=4096$, Adam, $\gamma=0.10$)
from $c_0$ to $c_{500}$, with the nearest training image to each final
state. The marked onset is the first cycle at which a trajectory
satisfies the $f_{\mathrm{mem}}$ criterion and remains classified as
memorized for the rest of the run.
(\textbf{B}) Score divergence $\mathcal{A}_1$ decreases as the
trajectories approach these states; the decline can begin before the
marked onset.
(\textbf{C}) Consecutive-state LPIPS cosine similarity
$S_{\mathrm{cyc}}^{(k)}$ increases and reaches a high plateau, consistent
with residence after the approach.
(\textbf{D--F}) The corresponding CIFAR-10 trajectories
($n=8192$, Adam, $\gamma=0.05$) and measurements.
Colors identify trajectories across panels. Dotted lines and markers
show their onsets. Divergence and similarity are displayed with a
centered 50-cycle moving average. Annotations above \textbf{A} and
\textbf{D} give the training step $\tau$, one-shot
$f_{\mathrm{mem}}$, and FID.
}
    \label{fig:attractor-grid}
\end{figure*}
\begin{figure}[ht]
    \centering
    \includegraphics[width=\linewidth]{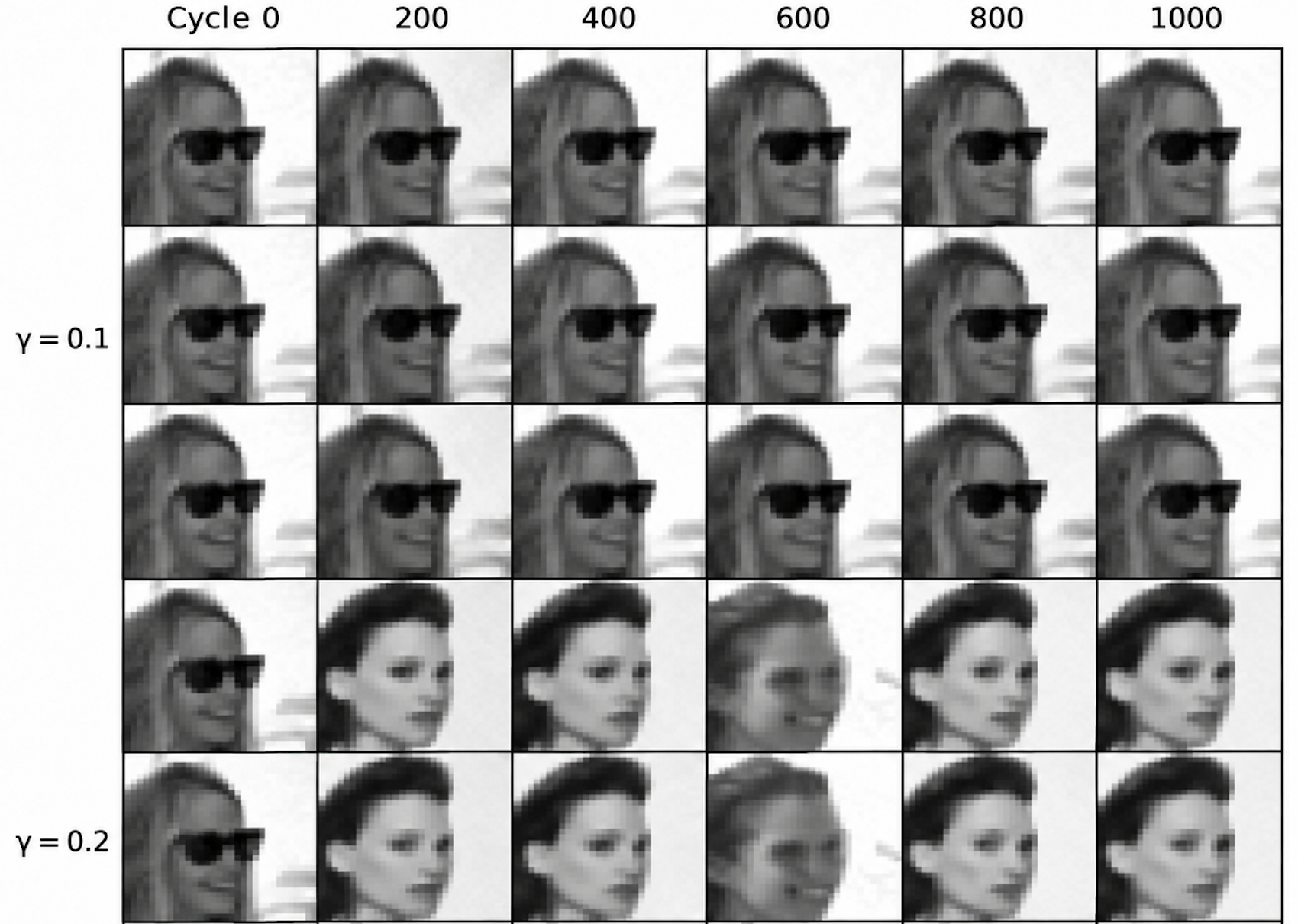}
    \caption{
        Basin residence and jumping under repeated cycling at
        $\tau = 40{,}000$
        (run: CelebA32, UNet, $n=4096$, Adam).
        Both sets begin in the basin associated with the same memorised
        training example. At $\gamma = 0.1$ (top), trajectories remain
        in this basin. At $\gamma = 0.2$ (bottom), stronger perturbations
        induce basin jumping, followed by sustained residence in another
        attractor, consistent with greater stability under cycling.
    }
    \label{fig:basin-jumping}
\end{figure}

\FloatBarrier

\subsection{Metrics}
\label{app:metrics}

\subsubsection{Memorization Fraction}
\label{app:metrics-fmem}

Following \citet{yoon2023diffusion,guMemorizationDiffusionModels2023,bonnaireWhyDiffusionModels2025}, we classify a
sample as memorized using its relative proximity to its two nearest
training neighbours. For a sample $x$, let $x_{(1)}$ and $x_{(2)}$
denote its nearest and second-nearest training neighbours under
Euclidean distance. Define
\begin{equation}
    r(x)=\frac{\|x-x_{(1)}\|_2}{\|x-x_{(2)}\|_2},
    \qquad
    m(x)=\mathbf{1}\!\left[r(x)<\kappa\right],
    \qquad \kappa=\frac{1}{3}.
\end{equation}
For an evaluated collection of $M$ samples, the memorization fraction
is the percentage satisfying this criterion:
\begin{equation}
    f_{\mathrm{mem}}
    =\frac{100}{M}\sum_{i=1}^{M}m(x_i).
\end{equation}
A small ratio indicates that a sample is substantially closer to one
training example than to its nearest competitor. The criterion does
not require exact equality with a training example.

For one-shot generation, we evaluate generated samples. During
cyclic denoising, we evaluate the trajectory states at each cycle.
For trajectories initialized from training samples, memorization
is measured against the full training set, not only the starting
image. An increase in $f_{\mathrm{mem}}$ therefore indicates that
more trajectories reach states satisfying the training-sample
memorization criterion; it does not measure the number of distinct
training samples reached.

We estimate uncertainty using $1{,}000$ bootstrap resamples of the
evaluated collection, each containing $M$ samples drawn with
replacement. The standard deviation of the bootstrap fractions
provides one standard error on $f_{\mathrm{mem}}$. 

\subsubsection{Cyclic Similarity}
\label{app:metrics-cyclic-similarity}

We measure perceptual similarity between consecutive saved trajectory
states using the VGG backbone of LPIPS. Images undergo the LPIPS
input preprocessing before feature extraction. Let $F_\ell(x)$
denote the feature map at layer $\ell$. We normalize its feature
vectors across channels at each spatial location, then flatten and
concatenate the normalized maps:
\begin{equation}
    \widetilde{F}_{\ell,h,w}(x)
    =
    \frac{F_{\ell,h,w}(x)}
    {\|F_{\ell,h,w}(x)\|_2+\epsilon},
    \qquad
    \Phi(x)
    =
    \operatorname{concat}_{\ell}
    \left[
        \operatorname{vec}\bigl(\widetilde{F}_{\ell}(x)\bigr)
    \right],
\end{equation}
where $\epsilon>0$ ensures numerical stability.
For trajectory $i$, let $x_i^{(k)}$ denote its state at the $k$th
saved observation. We define
\begin{equation}
    S_{\mathrm{cyc},i}^{(k)}
    =
    \frac{
        \left\langle
        \Phi(x_i^{(k)}),\Phi(x_i^{(k-1)})
        \right\rangle
    }{
        \|\Phi(x_i^{(k)})\|_2
        \|\Phi(x_i^{(k-1)})\|_2
    },
    \qquad
    S_{\mathrm{cyc}}^{(k)}
    =
    \frac{1}{M}\sum_{i=1}^{M}S_{\mathrm{cyc},i}^{(k)}.
\end{equation}
This statistic is cosine similarity in the
normalized LPIPS-VGG feature space.
Values close to one indicate small perceptual changes between saved
states. An increase followed by a sustained high plateau is consistent
with trajectories settling into basins that remain stable under the
cycling amplitude. This alone does not establish residence in a
training-sample basin, since degenerate attractors can also produce
high similarity. We therefore interpret $S_{\mathrm{cyc}}^{(k)}$
together with $f_{\mathrm{mem}}$ and score divergence.
For visualizations, we apply a centered moving average.

\subsubsection{Basin Volume}
\label{app:metrics-basin-volume}

We estimate the basin volume of a target sample
$\hat{x}\in\mathbb{R}^{d}$ using a deterministic DDIM recovery
sweep adapted from \citet{phamMemorizationGeneralizationEmergence2025}.
The sweep identifies the largest tested noise amplitude at which
forward-perturbed copies of $\hat{x}$ recover with sufficiently high
probability.

Let $Q_\gamma(\hat{x};\xi)$ denote the noise scheduler's forward
perturbation of $\hat{x}$ at amplitude $\gamma$, using an independent
draw $\xi\sim\mathcal{N}(0,I_d)$. Let
$\mathcal{D}^{\mathrm{DDIM}}_\theta(x_\gamma;\gamma)$ denote
deterministic DDIM denoising of $x_\gamma$ from that amplitude.
We test $\gamma\in\{0.05,0.10,\ldots,1.00\}$ in increasing order.
At each amplitude, we independently perturb $M=10$ copies of
$\hat{x}$, drawing fresh noise for every copy and every amplitude.

For copy $i$, write
\begin{equation}
x_\gamma^{(i)}
=
Q_\gamma(\hat{x};\xi_\gamma^{(i)}),
\qquad
\tilde{x}_\gamma^{(i)}
=
\mathcal{D}^{\mathrm{DDIM}}_\theta
\bigl(x_\gamma^{(i)};\gamma\bigr).
\end{equation}
Its recovery indicator and the empirical recovery probability are
\begin{equation}
R_\gamma^{(i)}
=
\mathbf{1}\!\left[
D_{\mathrm{rec}}
\bigl(\tilde{x}_\gamma^{(i)},\hat{x}\bigr)
\leq\epsilon_{\mathrm{rec}}
\right],
\qquad
p_{\mathrm{rec}}(\gamma)
=
\frac{1}{M}\sum_{i=1}^{M}R_\gamma^{(i)}.
\label{eq:basin-recovery-probability}
\end{equation}

For CelebA and CIFAR, $D_{\mathrm{rec}}$ is LPIPS with a VGG
backbone and $\epsilon_{\mathrm{rec}}=0.05$. Before evaluating
LPIPS, we invert the dataset normalization, clip pixel values to
$[0,1]$, and map them to $[-1,1]$. The grayscale CelebA images,
centre-cropped to $32\times32$, are replicated across three
channels for LPIPS; CIFAR images are already RGB. For the GMM
experiments, $D_{\mathrm{rec}}$ is Euclidean distance and
$\epsilon_{\mathrm{rec}}=0.4$.

The sweep stops at the first amplitude for which
$p_{\mathrm{rec}}(\gamma)<0.9$. We define the critical noise
amplitude $\gamma_c$ as the last preceding tested amplitude for
which $p_{\mathrm{rec}}(\gamma_c)\geq0.9$. Thus, at least nine of
the ten copies must recover for an amplitude to pass. We set
$\gamma_c=0$ if the first tested amplitude fails, and
$\gamma_c=1$ if every tested amplitude passes.

Once $\gamma_c$ is determined, we draw fresh noise
$\xi^\star\sim\mathcal{N}(0,I_d)$ and form a new perturbation
$x_{\gamma_c}=Q_{\gamma_c}(\hat{x};\xi^\star)$, with
$x_0=\hat{x}$. The ten copies in the recovery sweep determine
$\gamma_c$; they are not used to calculate the radius. We define
\begin{equation}
r_{\mathrm{rec}}(\hat{x})
=
\bigl\|\hat{x}-x_{\gamma_c}\bigr\|_2,
\qquad
V_{\mathrm{basin}}(\hat{x})
=
\frac{\pi^{d/2}}{\Gamma(d/2+1)}
\bigl[r_{\mathrm{rec}}(\hat{x})\bigr]^d.
\label{eq:basin-volume-estimate}
\end{equation}
Here, $\Gamma(\cdot)$ is the Euler gamma function.
The fresh perturbation $x_{\gamma_c}$ is used to measure
displacement and is not separately tested for recovery.
$V_{\mathrm{basin}}$ is an isotropic proxy of basin volume based on this estimated recovery amplitude.

\begin{algorithm}[t]
\caption{Recovery-based basin-volume estimate}
\label{alg:basin-volume}
\KwIn{Target $\hat{x}$; forward noising operation $Q$;
DDIM denoiser $\mathcal{D}^{\mathrm{DDIM}}_\theta$;
recovery distance $D_{\mathrm{rec}}$;
tolerance $\epsilon_{\mathrm{rec}}$}
\KwOut{$\gamma_c$, $r_{\mathrm{rec}}(\hat{x})$,
$V_{\mathrm{basin}}(\hat{x})$}

$M \gets 10$\;
$p_{\mathrm{stop}} \gets 0.9$\;
$\gamma_c \gets 0$\;

\For{$\gamma\in\{0.05,0.10,\ldots,1.00\}$}{
    \For{$i\gets 1$ \KwTo $M$}{
        Sample $\xi_i\sim\mathcal{N}(0,I_d)$\;
        $x_i\gets Q_\gamma(\hat{x};\xi_i)$\;
        $\tilde{x}_i\gets
        \mathcal{D}^{\mathrm{DDIM}}_\theta(x_i;\gamma)$\;
        $R_i\gets
        \mathbf{1}\!\left[
        D_{\mathrm{rec}}(\tilde{x}_i,\hat{x})
        \leq\epsilon_{\mathrm{rec}}
        \right]$\;
    }
    $p_{\mathrm{rec}}\gets
    \frac{1}{M}\sum_{i=1}^{M}R_i$\;
    \If{$p_{\mathrm{rec}}<p_{\mathrm{stop}}$}{
        \textbf{break}\;
    }
    $\gamma_c\gets\gamma$\;
}

\eIf{$\gamma_c=0$}{
    $x_{\gamma_c}\gets\hat{x}$\;
}{
    Sample fresh $\xi^\star\sim\mathcal{N}(0,I_d)$\;
    $x_{\gamma_c}\gets
    Q_{\gamma_c}(\hat{x};\xi^\star)$\;
}
$r_{\mathrm{rec}}(\hat{x})
\gets\|\hat{x}-x_{\gamma_c}\|_2$\;
$V_{\mathrm{basin}}(\hat{x})
\gets
\frac{\pi^{d/2}}{\Gamma(d/2+1)}
\bigl[r_{\mathrm{rec}}(\hat{x})\bigr]^d$\;
\textbf{return} $\gamma_c,\,
r_{\mathrm{rec}}(\hat{x}),\,
V_{\mathrm{basin}}(\hat{x})$\;
\end{algorithm}

\FloatBarrier

\subsubsection{Score Divergence and Its Estimation}
\label{app:score-divergence}

The pointwise divergence of the learned score field,
where defined, is $\nabla_x\cdot s_\theta(x,t)$.
We obtain the learned score from the model's noise
prediction $\epsilon_\theta$ as
\begin{equation}
s_\theta(x,t)
=
-\frac{\epsilon_\theta(x,t)}
{\sqrt{1-\bar{\alpha}_t}},
\label{eq:score-div-score}
\end{equation}
where $\bar{\alpha}_t$ is the cumulative
signal-retention coefficient of the forward noise
schedule. All quantities below are evaluated in the
data coordinates used by the diffusion model.

\paragraph{Evaluation point and boundary radius.}
For each original sample $x_0$, whether training,
held-out, or generated, we draw one forward perturbation
\begin{equation}
x_t
=
\sqrt{\bar{\alpha}_t}\,x_0
+
\sqrt{1-\bar{\alpha}_t}\,\epsilon,
\qquad
\epsilon\sim\mathcal N(0,I_d).
\label{eq:score-div-forward}
\end{equation}
The probing sphere is centred at $x_t$. Unless stated
otherwise, we use $t=1$, the final denoising timestep.
We calibrate the physical probing radius from all
$n_{\mathrm{train}}$ training samples:
\begin{equation}
R_{\mathrm p}
=
0.075\,
\operatorname{median}_{1\leq i\leq n_{\mathrm{train}}}
\left(
\min_{\substack{1\leq k\leq n_{\mathrm{train}}\\k\ne i}}
\|x_i-x_k\|_2
\right).
\label{eq:score-div-radius}
\end{equation}
The nearest-neighbour distances are computed between
unperturbed training samples in the model's data
coordinates. Thus the sphere probes the score field at
$7.5\%$ of the median training nearest-neighbour
distance. The implementation passes
a radius parameter equal to $R_{\mathrm p}/\sqrt d$
and multiplies it by $\sqrt d$ when forming boundary
points; their physical distance from $x_t$ is
$R_{\mathrm p}$.
The median nearest-neighbour distance provides a
data-dependent measure of typical training-sample
separation in the model's data coordinates. Setting
$R_{\mathrm p}$ to $7.5\%$ of this distance probes the
score field at a smaller spatial scale than typical
sample separation. A finite radius also retains a
measurable outward-flux signal: under local smoothness,
the mean flux is
$(R_{\mathrm p}/d)\nabla\cdot s_\theta(x_t,t)
+o(R_{\mathrm p})$ as $R_{\mathrm p}\to0$.
We use the same calibrated radius for training,
held-out, and generated samples throughout an
experiment, so their fluxes are compared at a common
spatial scale.

\paragraph{Reported score divergence.}
Following the boundary method of
\citet{brokmanIdentifyingMemorizationDiffusion2025},
we draw $J=64$ independent Gaussian vectors
$g_j\sim\mathcal N(0,I_d)$ and construct
\begin{equation}
e_j
=
\frac{g_j}{\|g_j\|_2},
\qquad
y_j
=
x_t+R_{\mathrm p}e_j.
\label{eq:score-div-boundary-points}
\end{equation}
Each $e_j$ is uniformly distributed on the unit sphere
and is the outward unit normal at $y_j$. We report the
unscaled mean outward learned-score flux
\begin{equation}
\mathcal A(x_0,t)
=
\frac{1}{J}\sum_{j=1}^{J}
s_\theta(y_j,t)\cdot e_j.
\label{eq:score-div-reported}
\end{equation}
We call this reported statistic \emph{score divergence}
for convenience and write
$\mathcal A_1(x_0)=\mathcal A(x_0,1)$.
The argument $x_0$ identifies the original sample;
the score evaluations are made on a sphere around
its forward-perturbed version $x_t$.

The relationship to mathematical divergence follows
when the score field satisfies the conditions for the
divergence theorem. Let $B_{R_{\mathrm p}}(x_t)$ be
the ball of radius $R_{\mathrm p}$ centred at $x_t$.
The sphere-area-to-ball-volume ratio is
$|\partial B_{R_{\mathrm p}}|/|B_{R_{\mathrm p}}|
=d/R_{\mathrm p}$, giving
\begin{equation}
\begin{aligned}
&\frac{1}{|B_{R_{\mathrm p}}|}
\int_{B_{R_{\mathrm p}}(x_t)}
\nabla_z\cdot s_\theta(z,t)\,\mathrm dz
\\
&\qquad=
\frac{d}{R_{\mathrm p}}\,
\mathbb E_{e\sim\mathrm{Unif}(\mathbb S^{d-1})}
\left[
s_\theta(x_t+R_{\mathrm p}e,t)\cdot e
\right].
\end{aligned}
\label{eq:score-div-flux-identity}
\end{equation}
Conditional on $x_t$,
Eq.~\eqref{eq:score-div-reported} is a Monte Carlo
estimate of the expectation on the right.
Multiplying it by $d/R_{\mathrm p}$ would give a
Monte Carlo estimate of the divergence averaged
over the ball. We report the unscaled value.
Because this factor is positive and fixed within
an experiment, omitting it preserves the signs
and ordering of samples compared within that
experiment. Numerical magnitudes from experiments
with different dimensions or radii do not have
the same scaling.

\paragraph{Normalized score divergence.}
Where $s_\theta(y_j,t)\ne0$, the unit-normalized
learned score is
\begin{equation}
\widehat{s}_\theta(y_j,t)
=
\frac{s_\theta(y_j,t)}
{\|s_\theta(y_j,t)\|_2}.
\label{eq:score-div-unit-score}
\end{equation}
We report its mean outward alignment,
\begin{equation}
\widehat{\mathcal A}(x_0,t)
=
\frac{1}{J}\sum_{j=1}^{J}
\widehat{s}_\theta(y_j,t)\cdot e_j.
\label{eq:score-div-normalized}
\end{equation}
We call this statistic \emph{normalized score
divergence} for convenience and write
$\widehat{\mathcal A}_1(x_0)
=\widehat{\mathcal A}(x_0,1)$.
Each summand lies in $[-1,1]$ where defined.
Negative values indicate that the boundary
scores point predominantly toward the centre
of the sphere.

\begin{algorithm}[t]
\caption{Monte Carlo computation of the reported score statistics}
\label{alg:score-divergence}
\KwIn{Sample $x_0\in\mathbb R^d$; noise predictor
$\epsilon_\theta$; schedule coefficient
$\bar{\alpha}_t$; calibrated radius
$R_{\mathrm p}$; timestep $t$}
\KwOut{$\mathcal A(x_0,t)$ and
$\widehat{\mathcal A}(x_0,t)$}

$J\gets64$\;
Sample $\epsilon\sim\mathcal N(0,I_d)$\;
$x_t\gets
\sqrt{\bar{\alpha}_t}\,x_0+
\sqrt{1-\bar{\alpha}_t}\,\epsilon$\;
$F\gets0$; $C\gets0$\;

\For{$j\gets1$ \KwTo $J$}{
    Sample $g_j\sim\mathcal N(0,I_d)$\;
    $e_j\gets g_j/\|g_j\|_2$\;
    $y_j\gets x_t+R_{\mathrm p}e_j$\;
    $s_j\gets
    -\epsilon_\theta(y_j,t)/
    \sqrt{1-\bar{\alpha}_t}$\;
    $F\gets F+s_j\cdot e_j$\;
    $C\gets C+
    (s_j\cdot e_j)/\|s_j\|_2$\;
}
$\mathcal A(x_0,t)\gets F/J$\;
$\widehat{\mathcal A}(x_0,t)\gets C/J$\;
\textbf{return}
$\mathcal A(x_0,t),\widehat{\mathcal A}(x_0,t)$\;
\end{algorithm}

For $n_{\mathrm{train}}=2048$, the physical probing
radii are $R_{\mathrm p}=0.102629$ for GMM,
$0.406240$ for CelebA, and $0.769862$ for CIFAR-10.
\FloatBarrier

\subsection{Membership Inference from Score Divergence}
\label{app:membership}
In Section~\ref{sec:exp-geometry} we saw that, as a model memorizes,
the normalized score divergence $\widehat{\mathcal{A}}_1$ around
training samples becomes markedly more negative than around held-out
test samples, while generated samples behave almost exactly like
training samples. This difference is large enough to decide, for a new
point, whether it belonged to the training set. Here we describe how we
turn this geometric signal into a simple membership inference test and
report its performance across training.

\paragraph{The membership score.}
For each candidate point $x$ we use the normalized score divergence
$\widehat{\mathcal{A}}_1(x)$, evaluated at a single small noise
amplitude at the final denoising step. Recall that
$\widehat{\mathcal{A}}_1$ is the mean outward alignment of the
unit-normalized score over a small sphere around $x$ (a mean of
cosines, bounded in $[-1,1]$); it is strongly negative where the score
field points inward, i.e.\ in the sharp basins that form around
memorized training samples. We use the same boundary estimator for both
datasets: we draw directions on a small sphere around $x$, evaluate the
learned score at those boundary points, normalize each score to unit
length, and average its alignment with the outward normal. For the GMM
the score is the exact learned score at $d=8$; for CelebA it is the
learned UNet score at $d=1024$. We flip the sign so that a higher score
means ``more likely a member'' and use this value directly as the
membership score.

\paragraph{No training labels required.}
The key property we exploit is that generated samples are statistically
indistinguishable from training samples under this score
(Figure~\ref{fig:membership-hist}). Generated samples can therefore
stand in for members whenever a reference distribution is needed, so
the procedure never requires access to the training set or its labels.

\paragraph{How we evaluate.}
We score an equal number of training samples (true members) and
held-out test samples (true non-members). We report the area under the
ROC curve (AUC), which is threshold-free and measures how well the
score separates the two groups on its own. We additionally report the
true positive rate (TPR) and false positive rate (FPR) at a moderate
$5\%$ false-positive budget, and the TPR at a strict $1\%$
false-positive budget (TPR@$1\%$FPR); all operating points are read
directly off the ROC curve, so no threshold needs to be chosen.

\paragraph{Results.}
Figure~\ref{fig:membership-hist} shows the distribution of the
membership score for training, test, and generated samples at
increasing training steps, for both GMM (top) and CelebA (bottom).
Early in training the three distributions overlap almost completely and
no rule can separate members from non-members. As training enters the
memorized regime, the training and generated distributions shift
together toward higher scores while the test distribution stays in
place, opening a clear margin. Tables~\ref{tab:membership-celeba}
and~\ref{tab:membership-gmm} quantify this: the AUC rises from close to
chance in the generalized regime to $1.0$ in the memorized regime, and
the TPR at $1\%$ FPR follows the same trend. The attack becomes
effective at the same training step at which the geometric signatures
of Section~\ref{sec:exp-geometry} first separate---well before
$f_{\mathrm{mem}}$ reflects memorization in the generated output.

\begin{figure}[H]
    \centering
    \includegraphics[width=\textwidth]{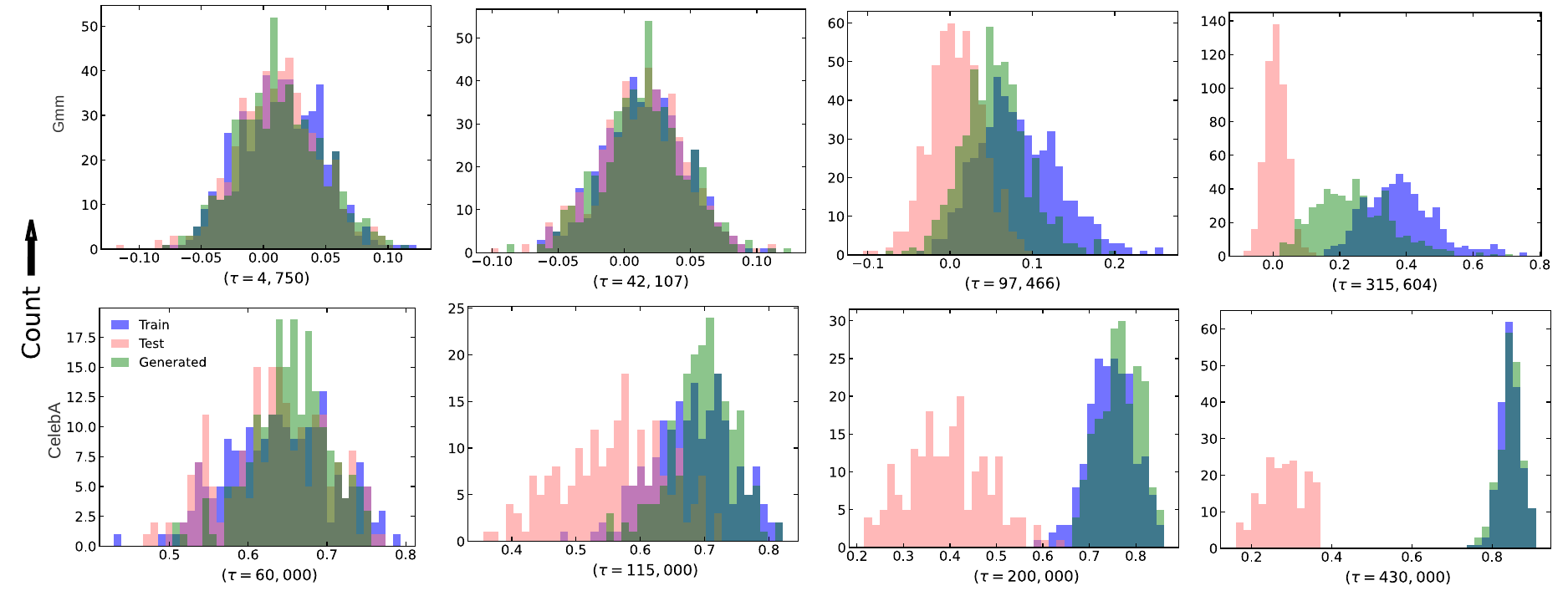}
    \caption{Distribution of the membership score (the sign-flipped
    normalized score divergence $\widehat{\mathcal{A}}_1$) for training
    (blue), held-out test (red), and generated (green) samples, at
    increasing training steps $\tau$ (left to right), for GMM (top row)
    and CelebA (bottom row). The three distributions overlap in the
    generalized regime and separate---with training and generated
    moving together, away from test---as the model memorizes.}
    \label{fig:membership-hist}
\end{figure}

\begin{table}[H]
\centering
\caption{Membership inference from the normalized score divergence
$\widehat{\mathcal{A}}_1$ at a single noise level, on \textbf{CelebA},
without any access to training labels. AUC is threshold-free. TPR and
FPR are reported at a moderate $5\%$ false-positive budget;
TPR@$1\%$FPR is the true-positive rate at a strict $1\%$ budget. All
operating points are read directly off the ROC curve.}
\label{tab:membership-celeba}
\begin{tabular}{rcccc}
\toprule
$\tau$ & AUC & TPR & FPR & TPR@1\%FPR \\
\midrule
60{,}000  & 0.522 & 0.875 & 0.835 & 0.035 \\
115{,}000 & 0.887 & 0.910 & 0.325 & 0.315 \\
200{,}000 & 1.000 & 0.925 & 0.000 & 0.995 \\
430{,}000 & 1.000 & 0.980 & 0.000 & 1.000 \\
\bottomrule
\end{tabular}
\end{table}

\begin{table}[H]
\centering
\caption{Membership inference from the normalized score divergence
$\widehat{\mathcal{A}}_1$ at a single noise level, on the \textbf{GMM}.
Conventions are identical to Table~\ref{tab:membership-celeba}.}
\label{tab:membership-gmm}
\begin{tabular}{rcccc}
\toprule
$\tau$ & AUC & TPR & FPR & TPR@1\%FPR \\
\midrule
4{,}750   & 0.531 & 0.066 & 0.047 & 0.012 \\
42{,}107  & 0.510 & 0.047 & 0.045 & 0.012 \\
97{,}466  & 0.909 & 0.607 & 0.049 & 0.490 \\
315{,}604 & 1.000 & 1.000 & 0.000 & 1.000 \\
\bottomrule
\end{tabular}
\end{table}

\subsection{Other Train Results}
\label{app:additional-train}
\begin{figure*}[ht]
    \centering
    \includegraphics[width=\linewidth]{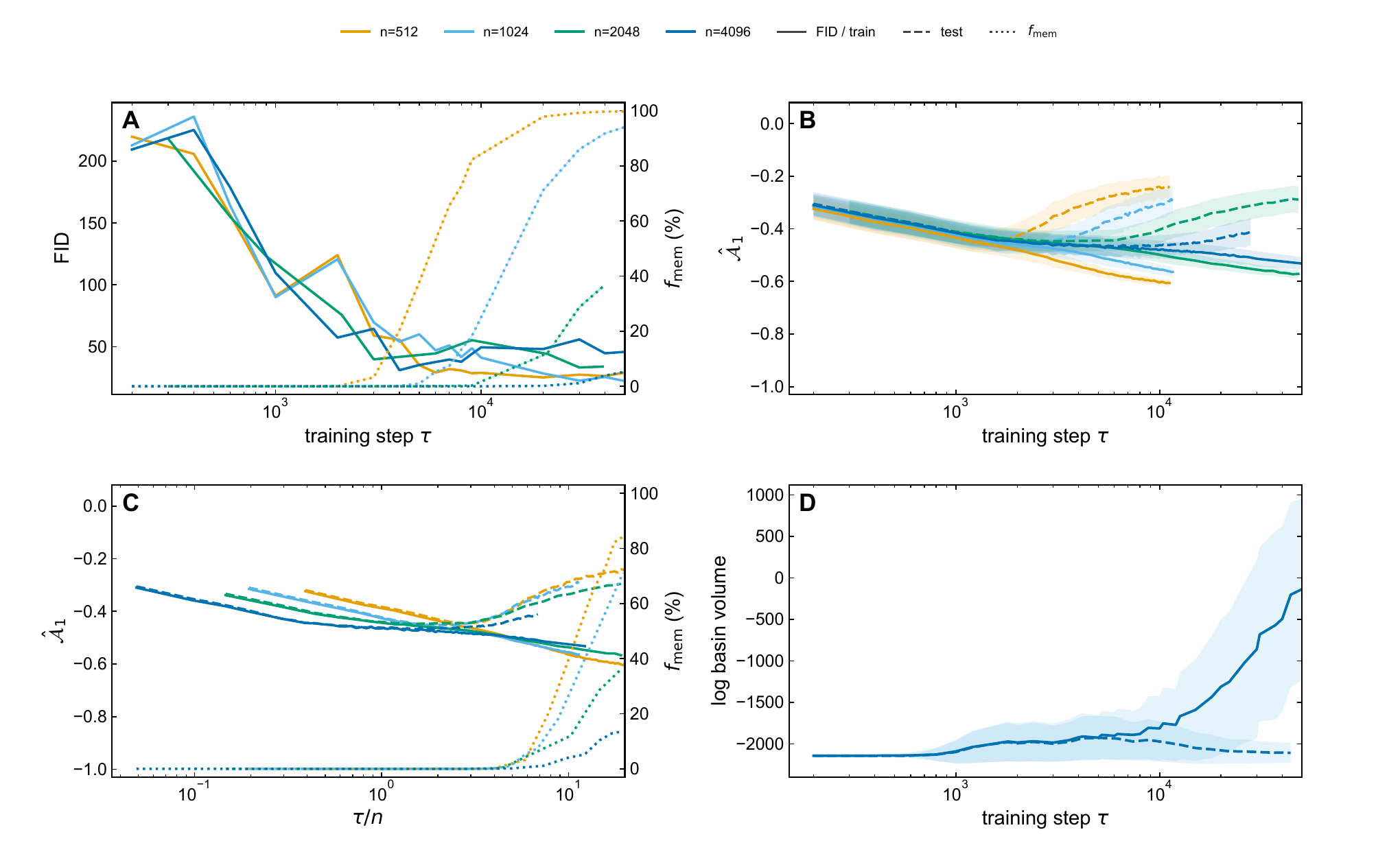}
 \caption{CelebA width-64 U-Net trained with Adam, number of training samples - \(n=512, 1024, 2048, 4096\). A: FID and memorization fraction, \(f_{\mathrm{mem}}\), versus training step \(\tau\). B: The normalized score divergence \(\hat{\mathcal A}_1\) for train and held-out test samples versus \(\tau\); bands show ±1 standard deviation across samples. Consistent lag between train-test \(\hat{\mathcal A}_1\) split and onset of $f_{mem}$ C: The same \(\hat{\mathcal A}_1\) curves and \(f_{\mathrm{mem}}\) versus normalized training steps \(\tau/n\). D: For the \(n=4096\) run, mean log basin-volume proxy for training and test images versus \(\tau\), with ±1 standard deviation. 
}
    \label{fig:adam-celeba}
\end{figure*}

\begin{figure}[ht]
    \centering
    \includegraphics[width=0.65\linewidth]{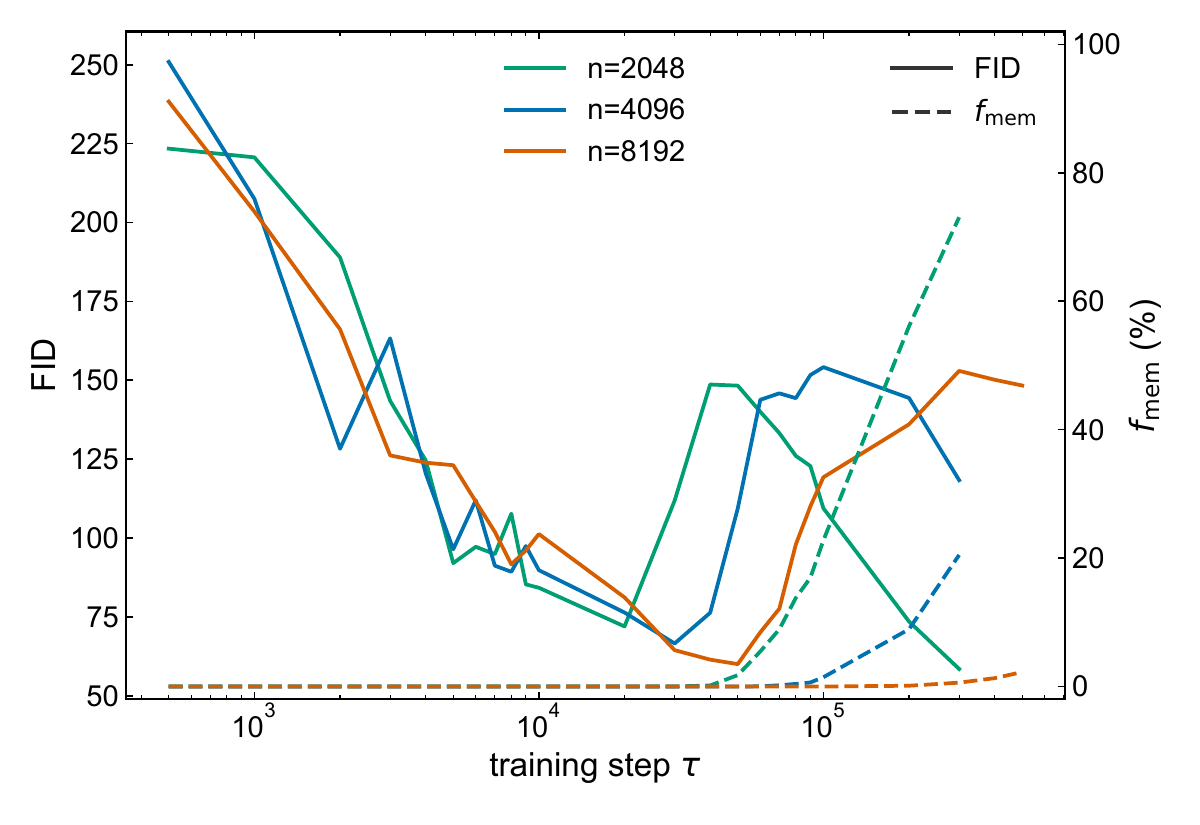}
 \caption{FID and memorization ratio of CIFAR-10 DiT-S runs with Adam optimizer and Cosine Noise Scheduler with n=2048, 4096,8192. 
}
    \label{fig:cifar-train}
\end{figure}
\FloatBarrier

\subsection{Visualizing the Score-Divergence Landscape}
\label{app:landscape-viz}

Figure~\ref{fig:landscape_cyclic}(a--b) visualizes the score-divergence
landscape on a two-dimensional slice through the high-dimensional
sample space. Because the score field lives in $d$ dimensions ($d=8$
for the GMM, $d=1024$ for CelebA), it cannot be plotted directly, so we
follow the low-dimensional projection approach of
\citet{bihani2024low}, originally developed for visualizing the energy
landscapes of atomic systems.

The idea is to pick two informative directions in the sample space and
evaluate the landscape on a grid spanned by them. We choose these
directions by principal component analysis, so that the plane we plot
captures as much of the relevant variation as possible rather than an
arbitrary slice. Concretely, we proceed in four steps:

\begin{enumerate}
    \item \textbf{Sample points.} We collect a set of points from the region of interest---training samples together with samples along their denoising trajectories---which trace out the part of the landscape we want to see.
    \item \textbf{Choose an origin and two directions.} We pick one
    point as the origin $x^\star$, form the displacement of every other
    sampled point from $x^\star$, and stack these displacements into a
    matrix. Its top two principal components $\mathbf{d}_1$ and
    $\mathbf{d}_2$ (the directions of largest variance) define the
    plane we plot.
    \item \textbf{Build a grid.} We form a grid of points
    $x(\alpha,\beta) = x^\star + \alpha\,\mathbf{d}_1 +
    \beta\,\mathbf{d}_2$ by varying the two interpolating coefficients
    $\alpha$ and $\beta$ over a fixed range.
    \item \textbf{Evaluate the landscape.} At every grid point we
    evaluate the score divergence $\mathcal{A}(x,t)$ and plot it as the
    height of the surface. Negative values (sharp, inward-pointing
    score) appear as wells, and near-zero values appear as flat regions.
\end{enumerate}

Choosing the plane by principal components ensures that the slice
reflects genuine structure in the landscape rather than an
uninformative random cut: the top two components account for most of
the variation among the sampled points, so the plotted plane passes
through the directions along which the landscape actually changes. The
resulting surface makes the difference between regimes directly
visible---broad and featureless in the generalized regime, and
punctuated by sharp wells around individual training samples in the
memorized regime.

\subsection{How the Score Function Evolves Across Training Regimes}
\begin{figure*}[!htbp]
    \includegraphics{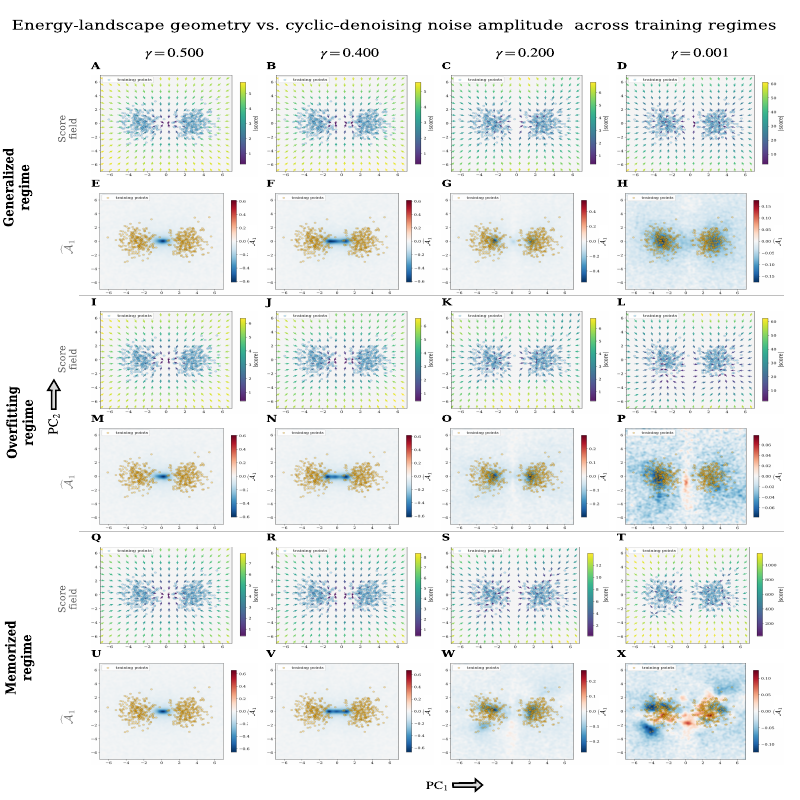}
    \caption{Visualised for the  GMM dataset - As $\gamma$ decreases from left to right, the two-basin structure sharpens from a single merged attractor into two well-separated basins. The generalization window (row 2) already shows some sign of attractor formation, but this sharpening becomes more pronounced and localized only in the memorized regime (row 3), compared to the generalized regime (row 1) -- directly quantifying its effect on the cyclic denoising dynamics.}
    \label{fig:score-func-gmm-gamma}
\end{figure*}

\begin{figure*}[ht]
    \centering
    \includegraphics[width = \textwidth]{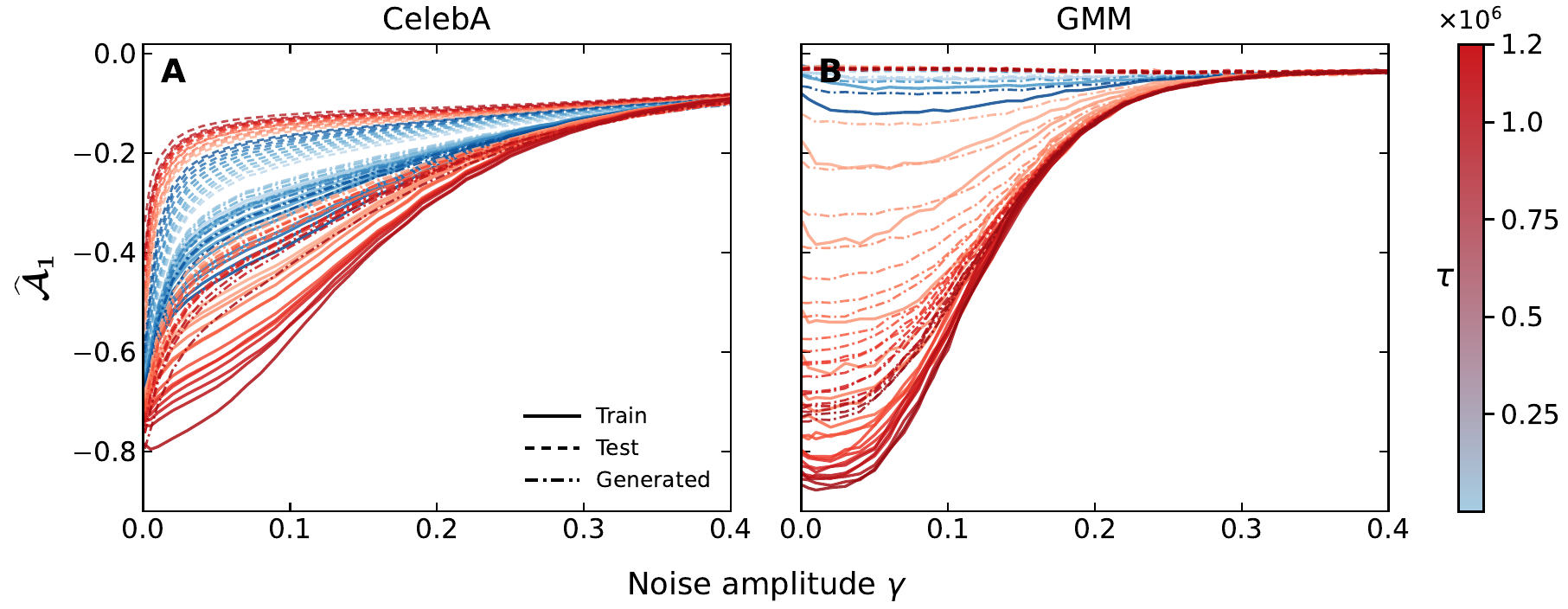}
    \caption{\textbf{Normalized score divergence versus noise amplitude
    $\gamma$ across all training checkpoints}, for CelebA
    \textbf{(a)} and the GMM \textbf{(b)}. Each curve is one checkpoint,
    coloured by training regime: blue for the generalized regime and
    orange for the memorized regime, with lighter to darker shades as the
    training step $\tau$ increases within each regime. Line style
    distinguishes train, test, and generated samples. In the generalized
    regime (blue) the curves are shallow and cluster together across all
    $\gamma$. In the memorized regime (orange) the normalized divergence is
    strongly negative at small $\gamma$ and rises toward zero as $\gamma$
    increases, with the effect deepening as $\tau$ grows.}
    \label{fig:regime-gamma-grid}
\end{figure*}
\FloatBarrier
\section{Finer Experiment Details}
\label{app:finer-things-club}
We train models from scratch on a synthetic Gaussian mixture,
CelebA, and CIFAR-10, continuing training until the memorization
fraction $f_{\mathrm{mem}}$ is high.
We evaluate checkpoints spanning pre-generalization,
generalization, and memorization.
For all three datasets, checkpoint sample generation uses
denoising diffusion implicit model (DDIM) sampling.

\paragraph{Synthetic data.}
We use an equally weighted two-component Gaussian mixture
in $d=8$ dimensions,
$P_0=\tfrac12\mathcal{N}(\boldsymbol{\mu},I_d)
+\tfrac12\mathcal{N}(-\boldsymbol{\mu},I_d)$,
with $\boldsymbol{\mu}=\mathbf{1}_d$.
The population score is available analytically.

The score network is a residual multilayer perceptron (MLP)
with three residual blocks and sinusoidal time embeddings. The standard model we used has $102{,}024$ parameters.
Training-set sizes are
$n\in\{128,256,512,1024,2048,4096\}$.
Models are trained using full-batch stochastic gradient descent
(SGD) with learning rate $6\times10^{-3}$ and momentum $0.95$,
for up to $4\times10^6$ optimization steps.
We use a linear noise schedule from $\beta_1=10^{-4}$
to $\beta_T=2\times10^{-2}$ over $T=1{,}000$ diffusion timesteps.
Generation uses DDIM with $200$ denoising steps.

\paragraph{CelebA.}
We centre-crop and downsample CelebA~\citep{7410782} to
$32\times32$ grayscale images, giving $d=1{,}024$.
We apply no data augmentation.
The score network is a denoising diffusion probabilistic model
(DDPM) U-Net with three resolution levels, channel multipliers
$\{1,2,4\}$, and attention at the two coarsest resolutions.
The standard model we used has $16$ million
parameters.

We train separate SGD and Adam models on subsets of size
$n\in\{512,1024,2048,4096\}$.
The SGD models use learning rate $0.01$, momentum $0.95$,
and batch size $B=\min(n,512)$, with no exponential moving
average (EMA) or learning-rate warm-up.
They are trained for at least $2\times10^6$ optimization steps.
The Adam models use learning rate $10^{-4}$,
$(\beta_1^{\mathrm{Adam}},\beta_2^{\mathrm{Adam}})
=(0.9,0.999)$, batch size $512$, and no weight decay.
They are trained for $2\times10^6$ optimization steps.

All models use a linear noise schedule from $\beta_1=10^{-4}$
to $\beta_T=2\times10^{-2}$ over $T=1{,}000$ diffusion timesteps.
Evaluation uses the same disjoint held-out set of $10{,}000$
images across all runs.
At each checkpoint, generation uses DDIM with $500$ denoising
steps.

\paragraph{CIFAR-10.}
We use CIFAR-10~\citep{krizhevsky2009learning} at its native
$32\times32$ RGB resolution, giving $d=3{,}072$.
We apply no data augmentation.
We train unconditional pixel-space diffusion transformers
(DiT-S/2)~\citep{Peebles2022ScalableDM}.
The architecture uses patch size $2$, hidden width $384$,
$12$ transformer blocks, $6$ attention heads, and an MLP
ratio of $4$, giving approximately $32.6$ million parameters.

We train separate models on subsets of size
$n\in\{2048,4096,8192\}$.
All models use Adam with learning rate $10^{-4}$,
$(\beta_1^{\mathrm{Adam}},\beta_2^{\mathrm{Adam}})
=(0.9,0.999)$, batch size $512$, and no weight decay.
They are trained for $2\times10^6$ optimization steps.
We use a cosine noise schedule
\citep{nicholImprovedDenoisingDiffusion2021}
with offset $s=0.008$ over $T=500$ diffusion timesteps.

Evaluation uses the same set of $7{,}952$ CIFAR-10 test images
across all runs.
The remaining $2{,}048$ test images are reserved for loss
evaluation.
At each checkpoint, we generate $10{,}000$ samples using DDIM
with $100$ denoising steps.

\paragraph{Stable Diffusion v1.4.}
We evaluate off-the-shelf Stable Diffusion v1.4 \citep{9878449} without retraining
or fine-tuning, in its $4\times64\times64$ latent space ($d=16{,}384$), with the
model's own scaled-linear noise schedule from $\beta_1=8.5\times10^{-4}$ to
$\beta_T=1.2\times10^{-2}$ over $T=1{,}000$ timesteps. Each chain is initialized
from a $50$-step DDIM generation at guidance scale $7.5$; every subsequent cycle
perturbs the latent to noise amplitude $\gamma$ and denoises it back with $10$ DDIM
steps. Chains run for $200$ cycles with $4$ seeds per prompt, and seeds are fixed
across $\gamma$, so that rows differing only in $\gamma$ share a $c_0$ image.

Prompts come from three sources. The individual chains of Fig.~\ref{fig:sd}A--C use
memorized prompts from \citep{hongMemBenchMemorizedImage2024}, which pairs
trigger prompts with the LAION images they reproduce; the six prompts used are
listed verbatim in Table~\ref{tab:sd-chains}. The population panels
Fig.~\ref{fig:sd}D--E use additional plain english prompts from \citep{hongMemBenchMemorizedImage2024} and prompts from \citep{webster2023reproducibleextractiontrainingimages}: In total $71$ completely
memorized (\textsc{mv}) and all $113$ non-memorized (\textsc{n}) prompts, together
with a fixed random subset of $150$ of the $339$ partially memorized (\textsc{tv})
prompts. Plain-English prompts, for which no training-image match exists, are the
first $150$ of the $499$ non-memorized prompts of MemBench. Where a reference
training image is available we measure $\mathrm{SSCD}$; this replaces
$f_{\mathrm{mem}}$, which requires access to the training set. Seed agreement
$\mathrm{SA}$ needs no reference image and is computed for every prompt. 

\begin{table}[H]\centering\small
\begin{tabular}{llrrrl}
\toprule
row & prompt & idx & $\gamma$ & seed\\
\midrule
\texttt{hippo}      & \makecell[l]{\texttt{Mothers influence on}\\[-2pt] \texttt{her young hippo}} & 1262 & 0.9 & 2\\
\texttt{laura}      & \texttt{insights with laura unknown}          &   67 & 0.1 & 2\\
\texttt{laura}      & \texttt{insights with laura unknown}          &   67 & 0.9 & 2\\
\texttt{air-con}    & \texttt{air conditioners set for}             &  215 & 0.9 & 1\\
\texttt{air-con}    & \texttt{air conditioners set for}             &  215 & 0.9 & 0\\
\texttt{label tape} & \makecell[l]{\texttt{set of nine brother tze 231}\\[-2pt]
\texttt{offlight white p-touch label}} & 662 & 0.9 & 2\\
\bottomrule
\end{tabular}
\caption{The six chains of Fig.~\ref{fig:sd}A, verbatim. \emph{idx} is the row index
in the MemBench SD v1.4 prompt table \citep{hongMemBenchMemorizedImage2024}, which
also supplies the matched training image. The two \texttt{laura} rows share a prompt, and seed differing only in $\gamma$; the two
\texttt{air-con} rows share a prompt and $\gamma$ and differ only in seed.}
\label{tab:sd-chains}
\end{table}

Trajectory-level measurements use a dense grid of $31$ cycles between $c_0$ and
$c_{200}$, with both $\mathrm{SSCD}$ and $\mathcal{A}_1$ evaluated at every one.
Population measurements use $15$ cycles for $\mathrm{SSCD}$ and $\mathrm{SA}$ and,
because the divergence probe dominates the cost, the subgrid
$\{0,10,50,100,200\}$ for $\mathcal{A}_1$; the $\gamma$ sweep covers $15$ amplitudes
from $0.005$ to $0.9$.

We estimate the unnormalized score divergence $\mathcal{A}$ with the forward-pass
boundary-flux estimator ($p{=}2$) of
\citet{brokmanIdentifyingMemorizationDiffusion2025}, probing at
$y=\sqrt{\bar\alpha_t}\,x$ on a sphere of radius $r=10^{-2}\sigma_t$ with $16$
directions arranged as $8$ antithetic pairs. We evaluate at $t=1$ and write
$\mathcal{A}_1$. Conditional and unconditional branches use the same probe points
and directions, and we report their difference $\Delta\mathcal{A}_1$.

For the detection benchmark of Appendix~\ref{app:SD_AUC} we instead cycle at
$\gamma=0.9$ for $100$ cycles with checkpoints at $c\in\{0,10,20,40,60,80,100\}$,
over $500$ memorized prompts from \citet{wenDetectingExplainingMitigating2024} as
positives and $2{,}000$ negatives drawn evenly from LAION, COCO-2017 val, Lexica,
and random CLIP-token strings, with $4$ seeds each and the same probe settings.

\paragraph{Computational resources.}
Experiments were run on NVIDIA A100 GPUs, 80\,GB GPU memory. UNet and DiT models were trained using data parallelism
across two GPUs per run. Cyclic denoising used one GPU per run;
trajectories evolved independently and were processed in batches.
The cyclic experiments spanned multiple checkpoints and noise
amplitudes on CelebA, CIFAR-10, and the GMM, with runs of up to
5{,}000 cycles. 

\section{Detection on Stable Diffusion}
\label{app:SD_AUC}

\begin{figure*}[ht]
\centering
\includegraphics[width=\textwidth]{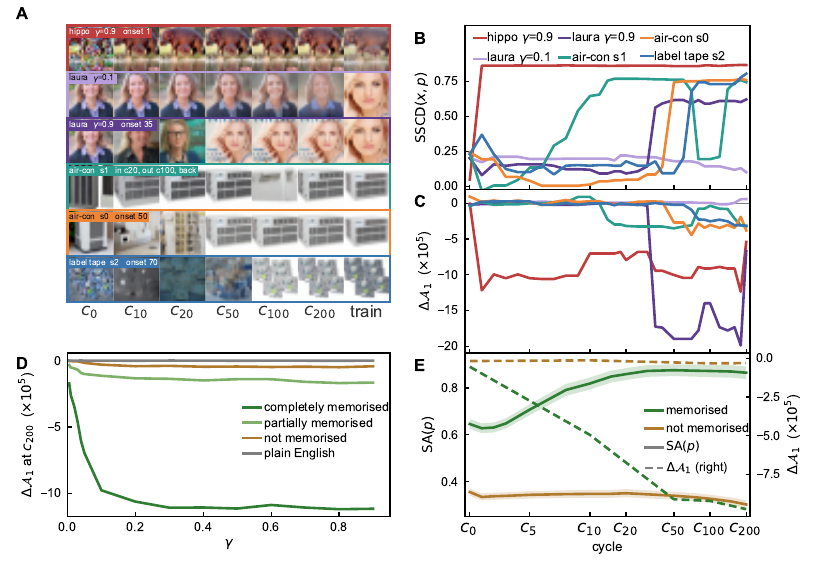}
\caption{Cyclic denoising on off-the-shelf Stable Diffusion v1.4. Each cycle noises
to amplitude $\gamma$ and denoises back; $c_k$ denotes the sample after $k$ cycles.
\textbf{(A)}~Six chains, one per row, with the matched LAION training image in the
\texttt{train} column; each row label gives the prompt, the setting that varies, and
the cycle at which the chain reaches its training image. The two \texttt{laura} rows
share a prompt, a seed and a $c_0$ image and differ only in $\gamma$; the
\texttt{air-con s1} row leaves its basin and re-enters it.
\textbf{(B,C)}~The same six chains per cycle: SSCD to the training image and
$\Delta\mathcal{A}_1$, on a shared axis. \textbf{(D)}~$\Delta\mathcal{A}_1$ after
$200$ cycles against $\gamma$, over $71$ completely memorised, $150$ partially
memorised and $113$ non-memorised prompts
\citep{webster2023reproducibleextractiontrainingimages} and $150$ plain-English
prompts \citep{hongMemBenchMemorizedImage2024}. \textbf{(E)}~Population means at
$\gamma=0.1$ over the $71$ completely memorised and $113$ non-memorised prompts,
$4$ seeds each: seed agreement (solid, left) and $\Delta\mathcal{A}_1$ (dashed,
right); bands are $\pm$SEM over prompts.}
\label{fig:sd}
\end{figure*}

\subsection{Data}
For below experiments and comparison with \citep{wenDetectingExplainingMitigating2024} we use the following prompts Positives: 500 memorised prompts from \cite{wenDetectingExplainingMitigating2024} Negatives: 2000 prompts, 500 each from
LAION, COCO-2017 val, Lexica, and random CLIP-token strings. Cycling at $\gamma=0.9$ with
10 DDIM steps per cycle; checkpoints at $c=0,10,20,40,60,80,100$.

\subsection{Uncertainty}
Intervals are from a stratified percentile bootstrap over prompts: positives and negatives
resampled independently with replacement, preserving the 500/2000 split, $B=4000$
resamples, 95\% percentile intervals. The same resample indices are reused across cycles and
across statistics, so all differences quoted are paired.

\subsection{Results by cycle}
\begin{table}[H]\centering\small
\begin{tabular}{rcc}
\toprule
cycle & AUC (95\% CI) & TPR@1\%FPR (95\% CI)\\
\midrule
0 & 0.811 [0.787, 0.836] & 0.526 [0.482, 0.572] \\
10 & 0.890 [0.864, 0.914] & 0.804 [0.766, 0.838] \\
20 & 0.935 [0.916, 0.953] & 0.834 [0.804, 0.872] \\
40 & 0.931 [0.910, 0.950] & 0.858 [0.824, 0.892] \\
60 & 0.944 [0.926, 0.961] & 0.866 [0.830, 0.900] \\
80 & 0.934 [0.914, 0.953] & 0.856 [0.820, 0.888] \\
100 & 0.940 [0.922, 0.957] & 0.852 [0.812, 0.884] \\
\bottomrule
\end{tabular}
\caption{$\Delta A$ at $\gamma=0.9$. Paired difference between $c=60$ and $c=40$:
$+0.016$, 95\% CI $[-0.010,+0.041]$.}
\end{table}

\begin{figure}[ht]\centering
\includegraphics{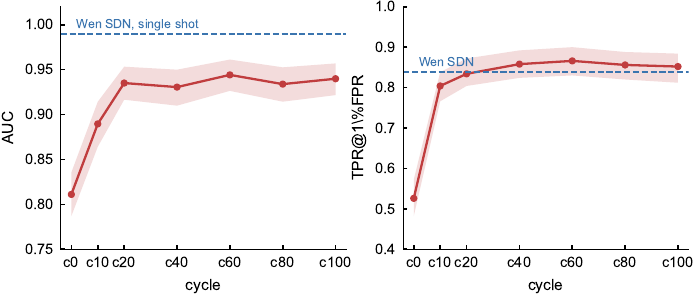}
\caption{AUC and TPR@1\%FPR against cycle count, with bootstrap bands.}
\end{figure}

\subsection{Comparison}
All rows use the same prompts and the same seeds.

\begin{table}[H]\centering\small
\begin{tabular}{lccc}
\toprule
statistic & AUC (95\% CI) & TPR@1\%FPR (95\% CI) & AP\\
\midrule
$\Delta A$  (ours) & 0.944 [0.926, 0.961] & 0.866 [0.830, 0.900] & 0.933 \\
$A_{\mathrm{cond}}$ alone & 0.948 [0.939, 0.957] & 0.194 [0.142, 0.246] & 0.782 \\
$A_{\mathrm{uncond}}$ alone & 0.899 [0.884, 0.914] & 0.066 [0.030, 0.094] & 0.623 \\
Wen SDN (single shot) & 0.990 [0.986, 0.993] & 0.838 [0.786, 0.882] & 0.970 \\
\bottomrule
\end{tabular}
\caption{At $c=60$, $\gamma=0.9$. Paired differences, $\Delta A$ minus SDN:
AUC $-0.0456$, 95\% CI $[-0.0633, -0.0287]$; TPR@1\%FPR $+0.0265$, 95\% CI $[-0.0260, +0.0841]$.}
\end{table}

\begin{figure}[H]\centering
\includegraphics{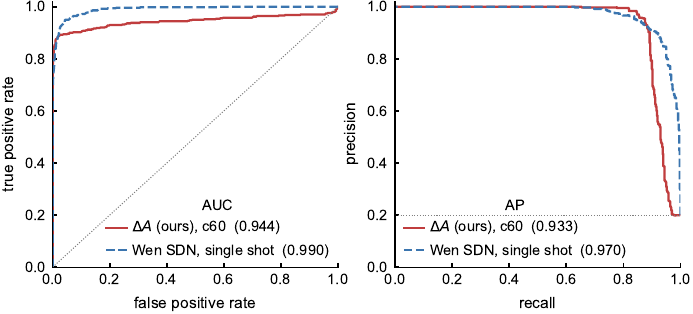}
\caption{ROC (left) and precision--recall (right) at $c=60$. Class balance is 1:4; the dotted
line in the right panel marks the positive rate.}
\end{figure}

\subsection{Raw divergence}
\begin{figure}[H]\centering
\includegraphics{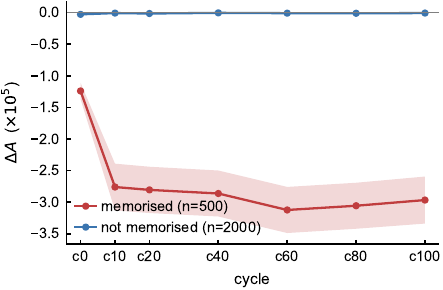}
\caption{Raw $\Delta A$ against cycle, population means with 95\% intervals. Memorised prompts
move from $-1.24\times10^{5}$ at $c=0$ to $-3.12\times10^{5}$ at $c=60$; non-memorised prompts
remain within $0.03\times10^{5}$ of zero.}
\end{figure}

\subsection{By negative source}
\begin{table}[H]\centering\small
\begin{tabular}{lrcccc}
\toprule
& & \multicolumn{2}{c}{$\Delta A$, $c=60$} & \multicolumn{2}{c}{Wen SDN, single shot}\\
\cmidrule(lr){3-4}\cmidrule(lr){5-6}
negatives & $n$ & AUC & TPR@1\%FPR & AUC & TPR@1\%FPR\\
\midrule
laion & 500 & 0.942 & 0.820 & 0.967 & 0.740 \\
coco & 500 & 0.945 & 0.884 & 0.996 & 0.950 \\
lexica & 500 & 0.945 & 0.892 & 0.997 & 0.976 \\
random & 500 & 0.944 & 0.876 & 1.000 & 0.998 \\
\bottomrule
\end{tabular}
\caption{Each row scores the 500 positives against one negative source.}
\end{table}

\end{document}